\documentclass[11pt]{article}
\usepackage[letterpaper, margin=1in]{geometry}

\usepackage[utf8]{inputenc} 
\usepackage[T1]{fontenc}    
\usepackage{hyperref}       
\usepackage{url}            
\usepackage{booktabs}       
\usepackage{amsfonts}       
\usepackage{nicefrac}       
\usepackage{microtype}      
\usepackage{xcolor}         
\usepackage{graphicx}

\usepackage{algorithm}
\usepackage{algorithmic}

\usepackage{amsmath}
\usepackage{amssymb}
\usepackage{mathtools}
\usepackage{amsthm}

\usepackage{times}
\usepackage[capitalize,noabbrev]{cleveref}
\usepackage{natbib}

\usepackage{dsfont}
\usepackage[mathscr]{euscript}
\usepackage{expl3}
\usepackage{colortbl}
\usepackage{thmtools}
\usepackage{thm-restate}
\usepackage{multirow}
\usepackage{wrapfig}
\usepackage{lipsum}
\usepackage{accents}

\usepackage{pifont}

\usepackage{MnSymbol}
\DeclareMathAlphabet\mathbb{U}{msb}{m}{n}
\usepackage{xpatch}

\def\Rset{\mathbb{R}}

\let\P\undefined

\DeclareMathOperator*{\P}{\mathbb{P}}
\DeclareMathOperator*{\E}{\mathbb E}

\DeclareMathOperator*{\argmin}{argmin}

\DeclareMathOperator*{\dom}{dom}

\DeclareMathOperator{\Ind}{\mathbb{I}}
\DeclareMathOperator*{\Int}{Int}
\DeclareMathOperator*{\Span}{span}

\DeclareMathOperator*{\Var}{Var}
\DeclareMathOperator*{\Cov}{Cov}

\DeclarePairedDelimiter{\abs}{\lvert}{\rvert}
\DeclarePairedDelimiter{\bracket}{[}{]}
\DeclarePairedDelimiter{\curl}{\{}{\}}
\DeclarePairedDelimiter{\paren}{(}{)}
\DeclarePairedDelimiter{\norm}{\|}{\|}
\DeclarePairedDelimiter{\tri}{\langle}{\rangle}

\ExplSyntaxOn
\tl_const:Nn \c_my_uc_alphabet_tl { ABCDEFGHIJKLMNOPQRSTUVWXYZ }
\tl_const:Nn \c_my_full_alphabet_tl { ABCDEFGHIJKLMNOPQRSTUVWXYZ
  abcdefghijklmnopqrstuvwxyz }

\tl_map_inline:Nn \c_my_uc_alphabet_tl
 { \cs_gset:cpn { c#1 } { \mathcal{#1} } }

\tl_map_inline:Nn \c_my_uc_alphabet_tl
 { \cs_gset:cpn { s#1 } { \mathscr{#1} } }

\tl_map_inline:Nn \c_my_full_alphabet_tl
 {
  \cs_gset:cpn { b#1 } { \mathbf{#1} }
  \cs_gset:cpn { sf#1 } { \mathsf{#1} }
 }
\ExplSyntaxOff

\newcommand{\h}{\widehat}
\newcommand{\ov}{\overline}
\newcommand{\wt}{\widetilde}
\newcommand{\e}{\epsilon}
\newcommand{\ignore}[1]{}

\newcommand{\KL}{\sfD_{\mathrm{KL}}}
\newcommand{\EX}{\sfE}

\newcommand{\algname}[1]{\textsc{#1}}

\hypersetup{
  breaklinks   = true,
  colorlinks   = true,
  urlcolor     = blue,
  linkcolor    = blue,
  citecolor   = blue
}

\usepackage[toc, page, header]{appendix}
\declaretheorem{theorem}
\newtheorem{lemma}[theorem]{Lemma}
\newtheorem{proposition}[theorem]{Proposition}

\newtheorem{corollary}[theorem]{Corollary}

\title{Generalized Distributional Alignment Games\\
  for Unbiased Answer-Level Fine-Tuning}
\author{
  Mehryar Mohri\\
  Google Research \& CIMS\\
  New York, NY 10011\\
  \texttt{mohri@google.com}
  \and
  Jon Schneider\\
  Google Research\\ 
  New York, NY 10011\\
  \texttt{jschnei@google.com}
  \and
  Yutao Zhong\\
  Google Research\\ 
  New York, NY 10011\\
  \texttt{yutaozhong@google.com}
}
\date{}

\begin{document}

\maketitle

\tableofcontents
\clearpage

\begin{abstract}
  The Distributional Alignment Game framework provides a powerful variational
  perspective on Answer-Level Fine-Tuning (ALFT). However, standard algorithms
  for these games rely on estimating logarithmic rewards from small batches,
  introducing a systematic bias due to Jensen's inequality that can destabilize
  training. In this paper, we systematically resolve this structural estimation
  bias. First, we generalize the alignment game to arbitrary Bregman
  divergences, showing that for a family of geometries inducing polynomial
  rewards, we can construct \emph{provably exact and unbiased} estimators using
  U-statistics. Second, for the canonical KL divergence game where an exact
  solution is impossible, we derive a globally robust minimax polynomial estimator that is \emph{provably optimal},
  achieving the fundamental statistical error limit of $\Theta(1/K^2)$, which we
  establish via the Ditzian-Totik theorem. Finally, we synthesize these two
  approaches to propose a novel Variance-Optimal Augmented Polynomial
  Optimization Program (AQP) Estimator, proving that by systematically reducing
  variance, our method achieves not only optimal bias but also \emph{provably
    accelerated game convergence}, leading to more efficient and stable
  training with zero online computational overhead.
\end{abstract}

\section{Introduction}
\label{sec:intro}

Recent approaches to Answer-Level Fine-Tuning (ALFT) have demonstrated
compelling empirical success in aligning Large Language Models (LLMs) within
complex reasoning domains, such as mathematics and code generation
\citep{Uesato2023, Lightman2023, Wang2022}. Unlike trace-level fine-tuning which
optimizes specific intermediate steps, ALFT focuses purely on the validity of
the final answer, granting the model the flexibility to discover diverse, valid
reasoning trajectories.

Recently, \citet*{MohriSchneiderWu2026} elegantly formulated ALFT as a
\emph{Distributional Alignment Game}. By introducing an auxiliary target
distribution and leveraging Fenchel duality, they successfully mapped the
computationally intractable marginalization over latent reasoning paths into a
tractable probability projection problem. This theoretical framework unified
various heuristic alignment goals—such as diversity-promoting inverse-frequency
penalties and coherence-based self-improvement—under a single mathematical
umbrella, guaranteeing that the game's Nash Equilibrium perfectly recovers the
optimal ALFT policy.

However, translating this continuous theoretical framework into a practical
reinforcement learning algorithm reveals a fundamental statistical
bottleneck. In scalable algorithms like Group Relative Policy Optimization
(GRPO) \citep{Shao2024}, the model evaluates its policy using a small, fixed
group size $K$ (e.g., $K=16$). For standard KL-regularized alignment games, the
theoretically optimal reward signal takes the form $R(z) = \log
\sfq(z)$. Since the true target marginal $\sfq$ is unknown, it must be
estimated empirically from the small batch. The strict concavity of the
logarithm ensures that $\E[\log \h{\nu}] < \log \E[\h{\nu}]$, introducing a
systematic $\cO(1/K)$ bias via Jensen's inequality.

This small-sample bias is not merely a theoretical artifact; it actively
destabilizes training. Since the magnitude of the error scales inversely with
the probability of the generated answer, this bias disproportionately penalizes
rare, exploratory outputs. It acts as a severe "anti-exploration penalty,"
artificially punishing the model for generating novel valid reasoning paths and
accelerating premature mode collapse.

Crucially, this estimation bias is not a flaw specific to GRPO, but a
fundamental challenge inherent to optimizing distributional objectives with
\emph{any} policy gradient (PG) algorithm, including standard methods like
REINFORCE or PPO. To apply a PG algorithm to a distributional objective (e.g.,
maximizing answer diversity), one must first shape a per-sample reward. For
objectives common in ALFT, this shaped reward is necessarily logarithmic. 

The structural bias is introduced at this reward estimation step, applying the
$\log$ function to a small-batch frequency, upstream of the PG algorithm
itself. The solutions in this paper are thus a necessary toolkit for correctly
applying \emph{any} policy gradient method to the important class of
distributional alignment problems. Furthermore, we emphasize that our bias solution is exceptionally efficient: replacing the biased empirical logarithm $\log(X/K)$ with a pre-computed array lookup $c^*_X$ is a \emph{computationally free} upgrade that improves RL stability and downstream performance with strictly zero runtime overhead.

How can we address this structural bias without resorting to computationally
prohibitive batch sizes? This paper presents a series of theoretical results
that form a complete toolkit for neutralizing the small-sample estimation bias
in ALFT. Our main contributions are as follows:

\begin{itemize}
\item \textbf{Exact Unbiased Estimation via Generalized Bregman Divergences:} We
  decouple the Distributional Alignment Game from the KL divergence,
  generalizing it to arbitrary Bregman divergences. We prove that for
  regularizers inducing polynomial reward mappings, we can construct
  \emph{exactly unbiased} estimators using U-statistics, completely eliminating
  the finite-sample bias.

\item \textbf{Optimal Low-Bias Estimation for KL Games:} For the canonical KL
  divergence, we establish that the gradient bias is fundamentally bounded from
  below by $\Theta(1/K^2)$ via the Ditzian-Totik theorem. We rigorously demonstrate 
  that local Taylor approximations fail uniformly at the boundary, strictly necessitating our 
  globally robust minimax polynomial estimator to achieve this theoretically optimal rate.

\item \textbf{Accelerated Convergence via Variance-Optimal Augmented Polynomial
    Optimization Program AQP Estimator:} We synthesize our findings to propose a
  novel Variance-Optimal AQP Estimator.  We prove that this estimator
  traces the exact Bias-Variance Pareto frontier to reduce variance and
  accelerate convergence.

\end{itemize}

\paragraph{Related Work.}
Our work intersects with several active areas in language model
alignment. \emph{Answer-Level and Outcome Supervision:} As reasoning tasks
become more complex, outcome-based supervision and ALFT have gained prominence
over process supervision \citep{Uesato2023, Lightman2023}. \emph{Preference
  Optimization and RLHF:} Methods like RLHF \citep{Ouyang2022} and DPO
\citep{Rafailov2023} are standard for trace-level alignment, with recent
advancements like GRPO \citep{Shao2024} extending them to group-based
settings. \emph{Distributional Alignment:} Our work builds directly on the
Distributional Alignment Game framework of \citet{MohriSchneiderWu2026}.

\emph{Estimation Bias in RL:}
The issue of logarithmic estimation bias from small batches has also been noted
in various entropy-regularized RL and mutual information estimation settings
\citep{Poole2019, Belghazi2018}. In modern RL, this bottleneck is particularly
severe when probability distributions are implicit, environmental, or
analytically intractable. Prime examples include applying maximum-entropy
regularization to Implicit Policies (like EBMs or discrete Diffusion) where
analytic probabilities are uncomputable, computing KL-divergence penalties in
Offline RL from limited dataset visitations, or estimating opponent entropy in
Multi-Agent RL from short observation windows. In these settings, the agent must
evaluate logarithmic functionals from a small batch of $K$ samples. The
techniques presented in this work offer a principled path toward debiasing such
estimates not just in ALFT, but across a wide range of low-sample RL domains.

\paragraph{Organization.}
The remainder of this paper is organized as follows. \cref{sec:preliminaries}
formally reviews the ALFT problem. In \cref{sec:bias_problem}, we rigorously
characterize the small-batch estimation bias. \cref{sec:poly_exact} introduces
our first solution: enabling exact unbiased estimation by generalizing the
game's geometry. \cref{sec:kl_poly_approx} returns to the canonical KL
divergence setting to introduce our second solution: the provably optimal
minimax polynomial estimator. \cref{sec:pareto} synthesizes these results into
a novel variance-optimal estimator that provably accelerates convergence. Finally,
\cref{sec:toolkit} provides practical guidance on choosing the right estimator
for a given task.

\section{Preliminaries: Answer-Level Fine-Tuning}
\label{sec:preliminaries}

Consider a sequence generation setting where $\sX$ denotes a set of problem descriptions or inputs, $\sY$ represents the domain of valid reasoning paths, and $\sZ$ is the set of terminal solutions. A deterministic parser $\EX \colon \sY \to \sZ$ translates any reasoning path into its corresponding final outcome. Given a context $x \in \sX$, a language model policy $\pi(y|x)$ generates intermediate reasoning trajectories, which inherently induces a marginal probability distribution over the final answers:
$
  \nu_\pi(z|x) = \sum_{y \in \EX^{-1}(z)} \pi(y|x).$

The core mathematical objective of Answer-Level Fine-Tuning (ALFT) is to identify a target policy $\pi$ that minimizes an outcome-based distributional penalty $\cR$ computed over the marginal $\nu_\pi$. To prevent catastrophic forgetting and maintain linguistic fluency, this is constrained by a trace-level divergence penalty relative to a reference model $\pi_0$. This manifests as a convex optimization problem over the policy simplex $\Pi \subseteq \Delta(\sY)^\sX$:
\begin{equation}
  \label{eq:primal_alft}
  \min_{\pi \in \Pi} \quad \cJ(\pi) = \E_{x} \bracket*{ \cR(\nu_\pi(\cdot|x))
    + \beta \KL(\pi(\cdot|x) \parallel \pi_0(\cdot|x)) },
\end{equation}
where $\beta > 0$ scales the trace-level regularization, and $\cR \colon \Delta(\sZ) \to \Rset$ is a lower semi-continuous (l.s.c.) convex functional describing the alignment target (e.g., maximizing answer entropy for exploration or minimizing it for consensus). Due to the strict convexity of the KL divergence penalty, the overall objective $\cJ(\pi)$ guarantees a unique optimal policy.

\textbf{The Intractability of Direct Optimization.} Calculating the exact gradient of Eq.~\eqref{eq:primal_alft} via backpropagation is computationally infeasible. Since $\cR$ evaluates the aggregated marginal $\nu_\pi$, obtaining the exact analytical gradient necessitates backpropagating through the summation over the pre-image $\EX^{-1}(z)$, which is combinatorially vast and entirely unobserved. Furthermore, standard Monte-Carlo estimates like REINFORCE suffer from devastating variance: assigning appropriate credit to a specific sampled trace $y$ fails to account for the unknown probability mass of countless other valid traces that could yield the identical final answer $z$. 

\textbf{The Game-Theoretic Solution.} To overcome this fundamental intractability, \citet*{MohriSchneiderWu2026} recently introduced the \emph{Distributional Alignment Game}. By leveraging Fenchel duality, they elegantly decoupled the marginalization process, introducing an auxiliary target distribution $\sfq \in \Delta(\sZ)$ to serve as a variational proxy. This transforms the primal ALFT problem into a tractable two-player min-max game between the generative policy $\pi$ and the target $\sfq$:
\begin{equation}
  \min_{\pi \in \Pi} \max_{\sfq \in \Delta(\sZ)} \quad \cG(\pi, \sfq) = \E_{x} \bracket*{ \beta \KL(\pi(\cdot|x) \parallel \pi_0(\cdot|x)) - \beta \E_{y \sim \pi(\cdot|x)} \bracket*{\log \sfq(\EX(y))} - \Psi(\sfq) },
\end{equation}
where $\Psi(\sfq)$ is the transformed Fenchel conjugate of the alignment goal $\cR$. This variational perspective successfully maps the computationally intractable marginalization over latent reasoning paths into a tractable probability projection problem.

The practical solution proposed in their framework relies on alternating best-response dynamics, typically implemented via Group Relative Policy Optimization (GRPO) \citep{Shao2024}. In each training iteration, the algorithm alternates between two distinct steps. First, in the \emph{Target Step}, the algorithm aggregates a sampled group of outputs to compute an optimal target distribution $\sfq^*$ that embodies the desired high-level alignment goal, for example, computing an arithmetic mean or mode to enforce consensus for self-improvement (coherence), or an inverse-frequency distribution to promote exploration (diversity). 

Second, in the \emph{Policy Step}, the target distribution is used to define an explicit per-sample reward. The theoretical Nash Equilibrium of the game dictates that the optimal reward for a generated trace $y$ is proportional to the log-likelihood of its final outcome $z = \EX(y)$ under the target distribution: $R(z) = \beta \log \sfq^*(z)$. The policy is then updated using GRPO advantages derived from this reward. By transforming the intractable marginalization over reasoning paths into a standard reinforcement learning projection, this solution enables models to optimize answer-level objectives efficiently. As we detail next, however, translating this theoretically continuous reward into a practical online algorithm by evaluating it on finite empirical batches reveals a critical statistical flaw.

\section{The Origin of Estimation Bias in Small-Batch ALFT}
\label{sec:bias_problem}

While the theoretical formulation of the Distributional Alignment Game perfectly
recovers the optimal answer-level policy in expectation, translating this
framework into a practical reinforcement learning algorithm introduces a
fundamental statistical bottleneck. This bottleneck arises directly from the
interaction between the non-linear dual mappings induced by standard
regularizers and the small-batch sampling regime typical of modern LLM training.

\subsection{Empirical Reward Estimation and Jensen's Gap}

In the generalized Distributional Alignment Game, the choice of the trace-level
regularizer and the corresponding distance-generating function $\Phi$ dictates
the dual mapping. This canonical link function maps the target distribution
$\sfq$ to the dual reward variable via some function $h$, such that:
\begin{equation}
  R(z) = \beta h(\sfq(z)),
\end{equation}
where $h(\sfq(z))$ corresponds to the partial derivative $\nabla \Phi(\sfq)_z$.

In practice, the true optimal target distribution $\sfq^*$ is unknown and must
be estimated dynamically. Algorithms like Group Relative Policy Optimization
(GRPO) operate on a small, fixed group of $K$ sampled traces per prompt
(typically $K \approx 16$). The target distribution is therefore approximated
using the empirical frequency of the answers within this small batch:
\begin{equation}
  \h{\sfq}(z) = \frac{\text{Count}(z)}{K}.
\end{equation}
Thus, the empirical reward signal used to update the policy is
$\h R(z) = \beta h(\h{\sfq}(z))$.

Since $h$ is generally a non-linear mapping, applying it to an empirical
stochastic estimate introduces a systematic discrepancy due to Jensen's
Inequality. For any strictly concave reward mapping (or strictly convex
penalty), the expected empirical reward systematically diverges from the true
reward evaluated at the expected marginal:
\begin{equation}
  \E[h(\h{\sfq}(z))] \neq h(\E[\h{\sfq}(z)]) = h(\sfq^*(z)).
\end{equation}

\subsection{Quantifying the \texorpdfstring{$\cO(1/K)$}{O(1/K)} Structural Penalty}

To understand the severity of this bias for any generic mapping $h$, we can
quantify it using a second-order Taylor expansion of the empirical reward around
the true expected marginal $\sfq^*(z)$.

Let the count of a specific answer $z$ be a Binomial random variable,
$X \sim \mathrm{Binomial}(K, \sfq^*(z))$, such that $\h{\sfq} = X/K$. The variance
of this estimator is
$\Var(\h{\sfq}(z)) = \frac{\sfq^*(z)(1-\sfq^*(z))}{K}$. Expanding
$h(\h{\sfq}(z))$ yields the expected empirical reward:
\begin{align}
  \E[\h R(z)]
  & = \beta \E[h(\h{\sfq}(z))] \nonumber \\
  &\approx \beta \paren*{ h(\sfq^*(z)) + \frac{1}{2} h''(\sfq^*(z))
    \Var(\h{\sfq}(z)) } \nonumber \\
  & = \beta h(\sfq^*(z)) + \frac{\beta}{2K} \sfq^*(z)(1-\sfq^*(z)) h''(\sfq^*(z)).
    \label{eq:general_taylor_bias}
\end{align}
This reveals a structural $\cO(K^{-1})$ estimation bias governed entirely by the
curvature $h''$ of the chosen dual mapping. In the asymptotic regime where
$K \to \infty$, this bias vanishes, and the empirical reward converges to the
true dual variable. However, ALFT algorithms rely on a small, fixed $K$.

\textbf{The KL Divergence Special Case.} The severity of this problem is most
apparent in the standard ALFT formulation, which uses the KL divergence. Here,
the mapping is the natural logarithm, $h(x) = \log x$, yielding a curvature of
$h''(x) = -1/x^2$. Substituting this into Eq.~\eqref{eq:general_taylor_bias}
gives the expected empirical reward for KL-regularized games:
\begin{equation}
  \E[\h R_{\text{KL}}(z)] \approx \beta \log \sfq^*(z)
  - \beta \frac{1 - \sfq^*(z)}{2K \sfq^*(z)}.
\end{equation}

\subsection{The Practical Impact on RL Dynamics}

For a fixed group size like $K=16$, the $\cO(1/K)$ bias induced by strongly
curved mappings like the logarithm is highly problematic for two reasons:
\begin{enumerate}
\item \textbf{Singularity at the Boundary:} For rare answers, the empirical
  frequency $\h{\sfq}$ is frequently zero. For maps like the logarithm, this
  makes the reward undefined ($-\infty$), necessitating ad-hoc Laplace smoothing
  that inherently flattens the learning signal and corrupts the theoretically
  optimal policy update.

\item \textbf{Anti-Exploration Penalty:} As seen in the KL case, the magnitude
  of the bias, $\frac{1-\sfq^*}{2K\sfq^*}$, is inversely proportional to the
  true probability mass $\sfq^*(z)$. This means the estimation error
  disproportionately penalizes rare, exploratory answers. Instead of receiving a
  neutral or theoretically sound signal, the policy receives a massive
  artificial negative gradient update for generating novel reasoning paths,
  actively driving the model toward premature mode collapse.
\end{enumerate}

To build scalable and stable ALFT algorithms, we must address this bias. The
following sections demonstrate how this can be achieved mathematically: first,
by changing the geometry of the game to admit polynomial mappings where exact
unbiased estimators exist, and second, by constructing global polynomial
approximations for the canonical KL-regularized objective.

\section{Bregman Divergences with Polynomial Reward and Exact
  Estimation}
\label{sec:poly_exact}

We can eliminate the estimation bias entirely by selecting a geometry for the
Distributional Alignment Game where the dual mapping function $h(\sfq)$ is a
bounded-degree polynomial.

\subsection{General Game via General Bregman Divergences}
\label{sec:bregman_general}

To systematically address the logarithmic bias, we generalize the Distributional
Alignment Game from the specific choice of the Kullback-Leibler (KL) divergence
to the broader class of Bregman divergences. Building on the framework of
\citet{MohriSchneiderWu2026}, this generalization is the key to our solution, as
it enables the use of geometries that induce polynomial reward functions. For
these functions, we demonstrate that it is possible to construct provably exact
and unbiased estimators, completely eliminating the estimation bias inherent in
the standard logarithmic case.

\subsubsection{Primal Problem and Dual Min-Max Formulation}
\label{sec:primal-dual-min-max}

Let $F \colon \Delta(\sY) \to \Rset$ be a function of \emph{Legendre type} in
the sense of \citet{Rockafellar1996}.  A proper, closed, and convex function $F$
is of Legendre type if it is differentiable on the interior of its domain,
$\Int(\dom(F))$, and its gradient mapping $\nabla F$ is a bijection
from $\Int(\dom(F))$ to $\Int(\dom(F^*))$, where $F^*$
is the Fenchel conjugate of $F$.

This assumption, which is standard in the theory of Bregman divergences, is the
precise condition needed for our framework for two critical reasons:
  \begin{enumerate}
  \item \textbf{Injectivity (Uniqueness):} The fact that $\nabla F$ is injective
    (one-to-one) is equivalent to $F$ being strictly convex. This ensures that
    the overall primal objective $\cJ_F(\pi)$ is strictly convex, guaranteeing a
    \emph{unique} optimal policy $\pi^*$. More importantly, it ensures that the
    inverse mapping, $(\nabla F)^{-1}$, is a well-defined single-valued
    function, making our closed-form solution for the optimal policy unique.

  \item \textbf{Surjectivity (Existence and Feasibility):} The fact that
    $\nabla F$ is surjective (onto) guarantees that its inverse,
    $(\nabla F)^{-1}$, is defined on the entire interior of the conjugate
    domain. For a valid reward gradient $\nabla\Phi(\sfq)$, the sum
    $\nabla F(\pi_0) + \nabla\Phi(\sfq)$ will lie in this domain. The bijection
    property ensures that applying $(\nabla F)^{-1}$ to this sum maps it to a
    unique point within the interior of the primal domain,
    $\Int(\Delta(\sY))$.
  \end{enumerate}
  This ensures that the optimal policy $\pi^*$ given in our theorem not only
  exists and is unique, but is also always a valid probability distribution with
  strictly positive probabilities, automatically satisfying the simplex
  constraints. The canonical negative entropy function,
  $F(\pi) = \sum \pi \log \pi$, which generates the KL divergence, is a primary
  example of a Legendre function for the probability simplex.

We redefine the primal Answer-Level Fine-Tuning (ALFT) problem as minimizing the
objective $\cJ_F(\pi)$:

\begin{equation}
  \label{eq:generalized_primal}
  \min_{\pi \in \Pi} \quad \cJ_F(\pi)
  = \E_{x} \bracket*{\cR(\nu_\pi(\cdot|x))
    + \beta \sfD_F(\pi(\cdot|x) \parallel \pi_0(\cdot|x))},
\end{equation}
where $\beta > 0$ is the regularization weight.

We now show how to transform this intractable marginalization problem into a
tractable game using the Fenchel Duality Theorem. Let
$\cR^* \colon \Rset^{|\sZ|} \to \Rset$ denote the Fenchel conjugate of $\cR$,
defined as
$\cR^*(u) = \sup_{\nu \in \Delta(\sZ)} \{ \tri{\nu, u} - \cR(\nu) \}$. Since
$\cR$ is convex and lower semi-continuous, the Fenchel-Moreau theorem guarantees
that $\cR$ equals its biconjugate $\cR^{**}$:
\begin{equation}
  \cR(\nu_\pi) = \sup_{u \in \Rset^{|\sZ|}} \curl*{ \tri*{\nu_\pi, u} - \cR^*(u) }.
\end{equation}

Substituting this into the primal problem, we decouple the marginalization
$\nu_\pi$. The game is naturally defined over the dual variables
$u \in \Rset^{|\sZ|}$. However, the objective is invariant under constant
shifts. Observe that for any functional $\cR$ defined on the probability simplex
$\Delta(\sZ)$, the Fenchel conjugate satisfies the shift property
$\cR^*(u + c\mathbf{1}) = \cR^*(u) + c$. This holds because
$\nu \in \Delta(\sZ)$ implies $\sum \nu(z) = 1$, so:
\begin{align*}
  \cR^*(u + c\mathbf{1})
  & = \sup_{\nu \in \Delta(\sZ)} \curl*{ \tri*{\nu, u + c\mathbf{1}}
    - \cR(\nu) } \\
  & = \sup_{\nu \in \Delta(\sZ)} \curl*{ \tri*{\nu, u} + c \sum_{z \in \sZ}
    \nu(z)
    - \cR(\nu) } = \cR^*(u) + c.
\end{align*}

Thus, the dual objective term $\tri{\nu_\pi, u} - \cR^*(u)$ is invariant under
the transformation $u \gets u + c\mathbf{1}$, as the linear shift $c$ cancels
exactly with the conjugate shift. This redundancy implies that the effective
dual space is the quotient space $\Rset^{|\sZ|} / \Span(\mathbf{1})$.

We can therefore parameterize the dual variables without loss of generality by
choosing a canonical representative from each equivalence class. We introduce a
strictly convex, differentiable answer-level distance-generating function
$\Phi \colon \Delta(\sZ) \to \Rset$. The gradient mapping $\nabla \Phi$
establishes a bijection between the interior of the simplex and the quotient
space. Specifically, we use the parameterization:
\begin{equation}
  u(z) = s \beta \nabla \Phi(q)_z,
  \quad \text{with } q \in \Int(\Delta(\sZ)),
\end{equation}
where $s = -1$ for coherence-promoting games (e.g., KL), and $s = +1$ for
diversity-promoting games.  Here, the scaling parameter $\beta$ is chosen
explicitly to offset the regularization weight of the Bregman divergence,
perfectly matching the dual space scaling. For any arbitrary vector
$u \in \Rset^{|\sZ|}$, there exists a unique scalar shift $c$ (analogous to the
shift generated by the Log-Sum-Exp function for KL divergence) such that the
shifted dual variables map to a valid normalized distribution $\sfq$.

We define the transformed conjugate functional
$\Psi(\sfq) = \cR^*(s \beta \nabla \Phi(\sfq))$. The coupling term over the trace
space expands as follows:
\begin{align*}
  \tri*{\nu_\pi(\cdot|x), s\beta \nabla \Phi(\sfq)}
  & = s\beta \sum_{z \in \sZ} \nu_\pi(z|x) \nabla \Phi(\sfq)_z \\
  & = s\beta \sum_{z \in \sZ} \paren*{ \sum_{y \in E^{-1}(z)} \pi(y|x) }
    \nabla \Phi(\sfq)_z \\
  & = s\beta \sum_{z \in \sZ} \sum_{y \in E^{-1}(z)} \pi(y|x) \nabla
    \Phi(\sfq)_{E(y)}
    \tag{Since $E(y)=z$ for $y \in E^{-1}(z)$} \\
  & = s\beta \sum_{y \in \sY} \pi(y|x) \nabla \Phi(\sfq)_{E(y)}
    \tag{Partition property: $\sY = \bigcup_z E^{-1}(z)$} \\
  & = s\beta \E_{y \sim \pi(\cdot|x)} \bracket*{ \nabla \Phi(\sfq)_{E(y)} }.
\end{align*}

Substituting this back into the formulation, we obtain the following min-max
objective function $\cG_F(\pi, \sfq)$:
\begin{equation}
  \label{eq:generalized_game}
  \cG_F(\pi, \sfq) \triangleq
  \E_{x} \bracket*{ \beta \sfD_F(\pi(\cdot|x) \parallel \pi_0(\cdot|x))
    + s\beta \E_{y \sim \pi} \bracket*{ \nabla \Phi(\sfq)_{E(y)} \mid x }
    - \Psi(\sfq) }.
\end{equation}

\subsubsection{Theoretical Consistency of the Generalized Game}
\label{sec:theoretical-consistency}

We now prove that the equilibrium of this generalized game recovers the exact
optimal solution to the ALFT primal problem, and we explicitly characterize the
structure of the optimal policy under an arbitrary Bregman divergence.

We will adopt the standard assumption that the spaces of inputs $\sX$, traces
$\sY$, and answers $\sZ$ are finite. This implies that the set of policies
$\Pi \subseteq (\Delta(\sY))^\sX$ and the set of target distributions
$\Delta(\sZ)$ are compact, convex sets in finite-dimensional Euclidean space.

\begin{theorem}[Consistency of the Generalized Game]
  \label{th:generalized_consistency}
  Let $\cR \colon \Delta(\sZ) \to \Rset$ be a convex, lower semi-continuous
  functional. Let $F$ and $\Phi$ be strictly convex, differentiable
  distance-generating functions for the trace and answer spaces,
  respectively. Then:
  \begin{enumerate}
  \item \textbf{Equivalence:} The primal problem \eqref{eq:generalized_primal}
    is equivalent to the min-max game:
    \begin{equation}
      \min_{\pi \in \Pi} \cJ_F(\pi)
      = \min_{\pi \in \Pi} \max_{\sfq \in \Delta(\sZ)} \cG_F(\pi, \sfq).
    \end{equation}

  \item \textbf{Optimal Policy Form:} For a fixed target distribution $\sfq$,
    the optimal policy $\pi^*$ is given by the Bregman projection of the
    reference policy $\pi_0$, shifted by the target reward gradient:
    \begin{equation}
      \label{eq:generalized_policy}
      \pi^*(\cdot|x)
      = \text{Proj}_{\Delta(\sY)}^F \paren*{ (\nabla F)^{-1}
        \big(\nabla F(\pi_0(\cdot|x)) - s \nabla\Phi(q(\cdot|x))_{E(\cdot)}\big) }.
    \end{equation}
  \end{enumerate}
\end{theorem}

\begin{proof}
  \textbf{Part 1 (Equivalence):} 
  First, we show that the primal objective $\cJ_F(\pi)$ can be written as a
  maximization problem. By the Fenchel-Moreau theorem, the convex functional
  $\cR$ is the biconjugate of its Fenchel conjugate $\cR^*$. Using our
  parameterization of the dual space, this gives:
  \begin{equation}
    \cR(\nu_\pi(\cdot|x))
    = \max_{q \in \Delta(\sZ)} \curl*{ s\beta \tri*{\nu_\pi(\cdot|x),
        \nabla \Phi(q)} - \Psi(q) }.
  \end{equation}
  Substituting this into the definition of the primal objective $\cJ_F(\pi)$
  from Eq.~\eqref{eq:generalized_primal}, we have:
  \begin{align*}
    \cJ_F(\pi)
    & = \E_x \bracket*{ \paren*{ \max_{q \in \Delta(\sZ)} \curl*{ s\beta \tri*{\nu_\pi, \nabla \Phi(q)} - \Psi(q) } } + \beta \sfD_F(\pi \parallel \pi_0) } \\
    & = \E_x \bracket*{ \max_{q \in \Delta(\sZ)} \curl*{ \beta \sfD_F(\pi \parallel \pi_0) + s\beta \tri*{\nu_\pi, \nabla \Phi(q)} - \Psi(q) } }.
  \end{align*}
  Using the definition of $\cG_F(\pi, q)$ from Eq.~\eqref{eq:generalized_game},
  where the inner expression is integrated over the expectation $\E_x$, we can
  write:
  \begin{equation}
    \label{eq:J_as_max_G}
    \cJ_F(\pi) = \max_{q \in \Delta(\sZ)} \cG_F(\pi, q).
  \end{equation}
  This identity holds by definition for any policy $\pi$. The primal problem is
  thus equivalent to solving
  $\min_{\pi \in \Pi} \max_{q \in \Delta(\sZ)} \cG_F(\pi, q)$.

  \textbf{Part 2 (Optimal Policy Form):} The proof for this part remains
  correct. It analyzes the inner minimization problem $\min_{\pi} \cG_F(\pi, q)$
  for a fixed $q$. By taking the functional derivative of the objective with
  respect to $\pi(y)$, we arrive at the first-order optimality condition:
  $\beta\nabla F(\pi^*(y)) - \beta\nabla F(\pi_0(y)) + s\beta\nabla \Phi(\sfq)_{E(y)} = 0$, which implies
  $\nabla F(\pi^*(y)) = \nabla F(\pi_0(y)) - s\nabla \Phi(\sfq)_{E(y)}$. Since
  $F$ is a Legendre function, the inverse map $(\nabla F)^{-1}$ exists, is
  unique, and maps its argument to a unique, valid probability distribution in
  $\Int(\Delta(\sY))$, thus yielding the closed-form solution in
  Eq.~\eqref{eq:generalized_policy}.
\end{proof}

\textbf{Simplex Projections.}
The behavior of the projection $\text{Proj}_{\Delta}^F$ depends strictly on the
chosen geometry $F$. When $F$ is the negative Shannon entropy (KL divergence),
the gradient diverges at the boundary. This ensures the unconstrained inverse
mapping automatically yields a strictly positive distribution, and the
projection resolves to a simple $\cO(1)$ normalizing constant
(Softmax). However, for polynomial geometries that remain finite at the boundary
(e.g., Euclidean distance), the unconstrained update may yield negative
values. Here, the exact optimal policy formally requires an explicit Bregman
projection onto the simplex (e.g., the Sparsemax operator
\citep{MartinsAstudillo2016}, which naturally induces exact sparsity in the
reasoning trajectories.

\textbf{Recovering the KL Game.}
The standard ALFT framework is recovered immediately as a special case. If we
choose the KL divergence, then $F(\pi) = \sum \pi \log \pi$ (negative entropy),
which implies $\nabla F(\pi) = \log \pi$ and $(\nabla F)^{-1}(x) = \exp(x)$. If
we similarly choose $\Phi(\sfq) = \sum \sfq \log \sfq$ for the answer space,
then $\nabla \Phi(\sfq)_z = \log \sfq(z)$ (modulo the constant $+1$ shift, which
maps to the quotient space equivalence). Substituting these into
Eq.~\eqref{eq:generalized_policy} (with $s=-1$ for the coherence-promoting KL game) yields
$\log \pi^*(y) = \log \pi_0(y) + \log \sfq(E(y))$, which simplifies to
$\pi^*(y) \propto \pi_0(y) \sfq(E(y))$, perfectly recovering the optimal policy
of the original KL-based ALFT formulation.

We complete the theoretical foundation by showing that finding an approximate
equilibrium of the generalized game yields a robust approximation to the primal
objective.

\begin{proposition}[Approximation Guarantee]
  Let $(\h{\pi}, \h{\sfq})$ be an $\e$-approximate equilibrium of the
  generalized game $\cG_F$, such that
  $\cG_F(\h{\pi}, \h{\sfq}) \le \min_{\pi} \max_{\sfq} \cG_F(\pi, \sfq) +
  \e$. Then $\h{\pi}$ is an $\e$-approximate minimizer of the primal objective
  $\cJ_F(\pi)$.
\end{proposition}

\begin{proof}
  By Part 1 of Theorem~\ref{th:generalized_consistency}, the primal objective
  satisfies $\cJ_F(\pi) = \max_{\sfq} \cG_F(\pi, \sfq)$. If
  $(\h{\pi}, \h{\sfq})$ constitutes an $\e$-equilibrium, then the maximum value
  achieved by fixing $\h{\pi}$ is bounded by the game value plus $\e$:
  \begin{equation}
    \cJ_F(\h{\pi})
    = \max_{\sfq} \cG_F(\h{\pi}, \sfq)
    \le \min_{\pi} \max_{\sfq} \cG_F(\pi, \sfq) + \e = \cJ_F(\pi^*)
    + \e.
  \end{equation}
  Thus, solving the game iteratively to an $\e$-tolerance rigorously bounds the
  sub-optimality of the resulting policy.
\end{proof}
This proof does not invoke the minimax theorem; it relies solely on the
point-wise identity $\cJ_F(\pi) = \max_q \cG_F(\pi, \sfq)$ for each fixed $\pi$,
which is established in Theorem~\ref{th:generalized_consistency} via the
Fenchel–Moreau theorem and requires no convex-concave structure beyond what is
already present in $\cG_F$.

\subsection{Theoretical Stability and Convergence via Dual Concavity}
\label{app:convergence_analysis}

While Theorem~\ref{th:generalized_consistency} establishes the mathematical
equivalence between the primal ALFT objective and the min-max game as an
identity, the practical stability of our iterative algorithms depends on the
objective's curvature. In the standard KL-regularized formulation, the objective
$\cG$ is not necessarily concave with respect to the target distribution $q$ on
the simplex. However, we can guarantee convergence by analyzing the game in the
dual parameter space.

\subsubsection{Well-Posedness in the Dual Space}
We define the dual variable
$u \in \Rset^{|\sZ|}/\text{span}(\mathbf{1})$ via the canonical
link function $u = s\beta \nabla \Phi(q)$. In this representation, the game
objective $\cG_F(\pi, u)$ takes a stable, structured form:
\begin{equation}
  \cG_F(\pi, u) \triangleq \E_x \bracket*{ \beta D_F(\pi \parallel
    \pi_0)
    + \langle \nu_\pi, u \rangle - \cR^*(u) } \text{.}
\end{equation}The convergence to a unique Nash Equilibrium $(\pi^*, u^*)$
is ensured by the following properties:
\begin{itemize}
\item \textbf{Strict Primal Convexity:} The objective is strictly convex in the
  policy $\pi$ due to the Legendre property of the trace-regularizer $F$. This
  ensures that the Bregman projection $(\nabla F)^{-1}$ is unique and maps to a
  valid probability distribution.

\item \textbf{Dual Concavity:} The coupling term $\langle \nu_\pi, u \rangle$ is
  strictly linear in $u$, and the transformed penalty $-\cR^*(u)$ is
  concave because the Fenchel conjugate $\cR^*$ is inherently convex.

\item \textbf{Bijective Mapping:} The Legendre property of the
  distance-generating function $\Phi$ ensures a bijection between the target
  distribution $q$ and the dual variable $u$.
\end{itemize}

\subsubsection{Algorithmic Implications}
The convex-concave structure in the $(\pi, u)$ space justifies two robust
solution paradigms:
\begin{enumerate}
\item \textbf{No-Regret Dynamics:} Algorithms such as Stochastic Mirror Descent
  can be applied to the dual variables $u$. With unbiased, bounded-variance
  gradients, the average iterates $(\ov \pi_T, \ov u_T)$ are guaranteed to
  converge to the Nash Equilibrium at a rate of $\cO(1/\sqrt{T})$ .

\item \textbf{Alternating Best Response:} Our primary algorithmic framework,
  \algname{Game-GRPO}, uses alternating updates. Since finding the optimal
  $q^* \in \Delta(\sZ)$ is mathematically equivalent to identifying the
  unique optimal dual variable $u^*$ that maximizes the concave dual objective,
  this procedure converges to the global optimum $(\pi^*, q^*)$.
\end{enumerate}
This dual perspective ensures that the Target Step always identifies a valid
direction for policy improvement, even when the underlying reward mapping $h(q)$
is highly non-linear.

\subsection{Characterization of Polynomial Reward Families}

For an exactly unbiased estimator to exist using a fixed group size $K$, the
target reward $R(z)$ must be a polynomial in $\sfq(z)$ of degree $d \le K$. This
imposes a strict structural condition on the choice of the Bregman divergence,
effectively characterizing the family of regularizers that permit exact,
small-batch optimization.

\begin{proposition}[Characterization of Polynomial Reward Families]
  \label{prop:poly_characterization}
  A Distributional Alignment Game yields a polynomial reward function of degree
  $d$ if and only if the distance-generating function $\Phi(\sfq)$ is a
  separable polynomial of degree $d+1$. Specifically,
  $\Phi(\sfq) = \sum_{z \in \sZ} \phi(\sfq(z))$, where:
  \begin{equation}
    \phi(x) = \sum_{m=1}^d \frac{c_m}{m+1} x^{m+1},
  \end{equation}
  and the coefficients $c_m$ are chosen such that the second derivative
  $\phi''(x) = \sum_{m=1}^d m c_m x^{m-1} > 0$ for all $x \in (0, 1]$, ensuring
  strict convexity on the simplex.
\end{proposition}

\begin{proof}
  By definition, the canonical link function defines the dual variable as
  $u(z) = s\beta \nabla \Phi(\sfq)_z$. Since the effective reward is $R(z) = -u(z) = -s\beta \nabla \Phi(\sfq)_z$, we require $R(z)$ to be a polynomial of
  degree $d$ in the marginal probability $\sfq(z)$, and critically, it must be a
  local scalar reward independent of $\sfq(z')$ for $z' \neq z$. Therefore, the
  gradient must take the form:
  \begin{equation}
    \frac{\partial \Phi(\sfq)}{\partial \sfq(z)} = \sum_{m=1}^d c_m \sfq(z)^m,
  \end{equation}
  for some coefficients $c_1, \dots, c_m$. Since the partial derivative with
  respect to $\sfq(z)$ depends only on $\sfq(z)$, the potential function $\Phi$
  must be strictly separable across the support $\sZ$, taking the form
  $\Phi(\sfq) = \sum_{z \in \sZ} \phi(\sfq(z))$.

  Integrating the gradient with respect to $\sfq(z)$ yields the primitive
  $\phi(x) = \sum_{m=1}^d \frac{c_m}{m+1} x^{m+1} + C$. The constant $C$ can be
  safely ignored as it vanishes under differentiation. Finally, for $\Phi$ to
  generate a valid Bregman divergence, it must be strictly convex over the
  probability simplex. This requires its Hessian to be positive definite,
  meaning the second derivative of the scalar function $\phi$ must be strictly
  positive on the domain $(0, 1]$, giving the condition
  $\phi''(x) = \sum_{m=1}^d m c_m x^{m-1} > 0$.
\end{proof}

This family corresponds to the Bregman divergences generated by integer-order
Tsallis entropies and their linear combinations.

\subsection{Exact Unbiased Estimation via U-Statistics}

If the game relies on a polynomial reward mapping of degree $d \le K$, we can
leverage U-statistics to construct an \emph{exactly unbiased} estimator using
the $K$ samples generated during the GRPO sampling phase.

\begin{theorem}[Exact Unbiased Polynomial Estimation]
  \label{th:unbiased_estimator}
  Let $X \sim \mathrm{Binomial}(K, \sfq(z))$ be a random variable counting the
  occurrences of an answer $z$ in a sampled group of size $K$. For any integer
  $m \le K$, the falling factorial statistic:
  \begin{equation}
    \h \sfq^m(z) = \frac{X(X-1)\cdots(X-m+1)}{K(K-1)\cdots(K-m+1)}
  \end{equation}
  is the unique, minimum-variance unbiased estimator for $\sfq(z)^m$.
\end{theorem}

\begin{proof}
  Let $p = \sfq(z)$. We evaluate the expectation of the numerator, the falling
  factorial $X^{\underline{m}} = X(X-1)\cdots(X-m+1)$, over the Binomial
  distribution:
  \begin{align}
    \E[X^{\underline{m}}]
    & = \sum_{k=m}^K k(k-1)\cdots(k-m+1) \binom{K}{k} p^k (1-p)^{K-k}
      \nonumber \\
    & = \sum_{k=m}^K \frac{k!}{(k-m)!} \frac{K!}{k!(K-k)!} p^k (1-p)^{K-k}
      \nonumber \\
    & = \frac{K!}{(K-m)!} p^m \sum_{k=m}^K \frac{(K-m)!}{(k-m)!(K-k)!}
      p^{k-m} (1-p)^{K-k}.
  \end{align}
  By applying a change of variables $j = k-m$, we recognize the remaining sum as
  the expansion of the binomial $(p + (1-p))^{K-m}$:
  \begin{align}
    \E[X^{\underline{m}}]
    & = K^{\underline{m}} p^m \sum_{j=0}^{K-m} \binom{K-m}{j} p^j (1-p)^{K-m-j}
      \nonumber \\
    & = K^{\underline{m}} p^m (p + 1 - p)^{K-m} = K^{\underline{m}} p^m.
  \end{align}
  Dividing both sides by the constant $K^{\underline{m}}$ yields
  $\E[\h \sfq^m(z)] = p^m = \sfq(z)^m$. Since this statistic is a function of
  the complete sufficient statistic $X$, the Lehmann-Scheff\'e theorem
  \citep{LehmannCasella1998,LehmannScheffe1950,LehmannScheffe1955}.  guarantees
  it is the unique minimum-variance unbiased estimator.
\end{proof}

By linearity of expectation, any polynomial reward
$R(z) = -s\beta \sum_{m=1}^d c_m \sfq(z)^m$ admits an exactly unbiased empirical
estimator $\h R(z) = -s\beta \sum_{m=1}^d c_m \h \sfq^m(z)$. Since we
can construct exactly unbiased estimators for these polynomial games, stochastic
alternating best-response dynamics are guaranteed to converge to the exact Nash
Equilibrium without any residual asymptotic bias.

\subsection{Computational Complexity}

A critical practical advantage of the polynomial alignment game is its
computational efficiency. Evaluating the exact unbiased estimator requires
computing the falling factorial for each unique answer $z$ in the sampled group.

For a polynomial reward of degree $d$, calculating the highest-order U-statistic
$\h \sfq^d(z)$ requires exactly $d-1$ scalar multiplications and $d$
subtractions per answer. Since $d$ is structurally bounded by the group size
$K$ (and typically $d \in \{1, 2\}$ in practice), the arithmetic complexity is
exactly $\cO(d)$ per evaluated trace. Across a full batch of size $K$, the total
overhead for computing the unbiased advantages is $\cO(dK)$.

This cost is strictly cheaper than evaluating transcendental functions like the
logarithm, which are required for standard KL-regularized updates. Since this
$\cO(dK)$ overhead is overwhelmingly dominated by the
$\cO(L \cdot \text{dim}^2)$ cost of the LLM forward and backward passes (where
$L$ is sequence length), exact unbiased estimation introduces essentially no
computational bottleneck to the training pipeline.

\subsection{Examples of Polynomial Alignment Games}

\textbf{Example 1: The Collision Penalty and the Euclidean Game ($d = 1$).}
A prominent objective for diversity promotion is minimizing the collision
probability of answers, defined as the strictly convex functional
$\cR(\nu) = \frac{1}{2}\sum_z \nu(z)^2$. At the trace level, the natural
regularizer corresponding to this geometry is the Squared Euclidean distance,
$D_F(\pi \parallel \pi_0) = \frac{1}{2}\norm{\pi - \pi_0}_2^2$.

To formulate this as a valid polynomial game satisfying Proposition 3, we select
the strictly convex distance-generating function
$\Phi(q) = \frac{1}{2}\sum_z q(z)^2$. The gradient is exactly linear:
$\nabla\Phi(q)_z = q(z)$. Since this is a diversity game ($s = +1$), the dual
penalty mapped to the policy is $u(z) = s\beta q(z)$. Because the primal policy
minimizes $\langle \nu_\pi, u \rangle$, the effective reward signal is
$-u(z) = -s\beta q(z) = -\beta q(z)$. Since the game is invariant to constant shifts (as
established in Section~\ref{sec:primal-dual-min-max}), we can add $\beta$ to
yield a strictly positive, diversity-promoting reward: $R(z) = \beta(1 - q(z))$.

By applying Theorem~\ref{th:unbiased_estimator} for $d=1$, the exactly unbiased
estimator for any group size $K \ge 1$ simply substitutes the true probability
with the empirical frequency $X/K$. Thus, the unbiased empirical reward is:
\[
  \h R_{\text{euclid}}(z) = \beta \paren*{1 - \frac{X}{K}}.
\]
This perfectly recovers a stable, rigorously unbiased diversity penalty that
explicitly satisfies all strict convexity requirements of the generalized game.

\textbf{Example 2: The Quadratic Reward Game ($d=2$).}  We can easily construct
games that demand more non-linear penalties while preserving unbiasedness. If we
choose the distance-generating function
$\Phi(\sfq) = \frac{1}{3} \sum \sfq(z)^3$ (closely related to Tsallis entropy of
order 3), the mapping becomes quadratic: $R(z) = \beta \sfq(z)^2$.

Using our formulation, the exact unbiased reward estimator requires a minimum
group size of $K \ge 2$ and is given by:
\begin{equation}
  \h R_{\text{quad}}(z) = \beta \frac{X(X-1)}{K(K-1)}.
\end{equation}
This provides a steeper penalty curve than the Gini index while remaining
perfectly unbiased.

\section{KL Divergence, Polynomial Approximation, and the \texorpdfstring{$\cO(1/K^2)$}{O(1/K\textasciicircum 2)} Limit}
\label{sec:kl_poly_approx}

While changing the trace-level regularizer to generate polynomial reward
mappings yields perfectly unbiased estimators, the KL divergence remains the
canonical choice for language model alignment. Under KL divergence, the link
function is the logarithm: $R(z) = \beta \log \sfq(z)$.

Since $\log(x) \to -\infty$ as $x \to 0$, it is impossible to uniformly
approximate the raw reward mapping with a bounded-degree polynomial on the
interval $[0,1]$. This implies that constructing a globally unbiased estimator
for the KL game is impossible. However, the structure of policy gradient
optimization presents a unique opportunity to neutralize this singularity by
shifting the approximation target.

\subsection{The Gradient-Weighted Objective and the Minimax Polynomial}
\label{sec:minimax-polynomial}

In policy gradient algorithms like GRPO, the model evaluates the gradient, which
weights the expected reward by the probability of the trace being generated. Let
$p = \sfq^*(z)$ be the true target probability. The expected bias of the
gradient step scales with:
\begin{equation}
  p \cdot \text{Bias}(p) = p \paren*{ \E[\h R(X)] - \beta \log p }
  = p \E[\h R(X)] - \beta p \log p.
\end{equation}
Since $\lim_{p \to 0} p \log p = 0$, the function $g(p) = p \log p$ is
continuous and bounded on the entire closed interval $[0,1]$. The singularity
disappears entirely, making a global polynomial approximation computationally
well-posed.

Any estimator $\h R(X)$ that acts on $K$ empirical samples drawn from a
Binomial distribution $X \sim \mathrm{Binomial}(K, p)$ is completely defined by
the $K+1$ values it assigns to the possible outcomes $\{0, 1, \dots, K\}$. Let
$\mathbf{c} = (c_0, c_1, \dots, c_K) \in \Rset^{K+1}$ such that
$\h R(k) = c_k$. The expected value of this estimator is strictly a polynomial
of degree $K$:
\begin{equation}
  P_{\mathbf{c}}(p) = \E[\h R(X)] = \sum_{k=0}^K c_k \binom{K}{k} p^k (1-p)^{K-k}.
\end{equation}
Thus, the gradient-weighted expectation $p P_{\mathbf{c}}(p)$ is a
polynomial of degree $K+1$ with a root at $p=0$.

We can therefore formulate the search for the optimal estimator as a minimax
polynomial approximation problem. We seek the vector of coefficients
$\mathbf{c}^*$ that minimizes the worst-case gradient bias across the entire
probability simplex:
\begin{equation}
  \label{eq:minimax_poly}
  \mathbf{c}^*
  = \argmin_{\mathbf{c} \in \Rset^{K+1}} \max_{p \in [0,1]} \, \abs[\big]{ p
    P_{\mathbf{c}}(p)  - \beta p \log p }.
\end{equation}

\subsection{Algorithmic Solution: Extracting the Minimax Estimator via Linear
  Programming}
\label{sec:minimax-lp}

While equation \eqref{eq:minimax_poly} defines a continuous minimax optimization
over the interval $[0,1]$, it can be efficiently solved by reformulating it as a
semi-infinite linear program. Furthermore, because the target function scales
linearly with $\beta$, we can factor out the hyperparameter and solve for the
normalized optimal coefficients $\wt{\mathbf{c}}^*$, where
$\mathbf{c}^* = \beta \wt{\mathbf{c}}^*$.

Notice that the expected value of the estimator is expressed in the Bernstein
polynomial basis. Let $B_{k, K}(p) = \binom{K}{k} p^k (1-p)^{K-k}$ denote the
$k$-th Bernstein basis polynomial of degree $K$. The gradient-weighted
expectation is exactly $p \sum_{k=0}^K \wt c_k B_{k, K}(p)$.

Since the optimization variables $\wt c_k$ appear linearly, we can
discretize the domain $p \in [0,1]$ into a dense grid of $M$ points,
$\cP = \{p_1, p_2, \dots, p_M\}$, to approximate the continuous worst-case
bound. By introducing an auxiliary variable $\e$ to represent the maximum
absolute error, the minimax problem is exactly formulated as the following
standard Linear Program (LP):
\begin{equation}
\label{eq:minimax-lp}
  \begin{aligned}
    & \min_{\wt{\mathbf{c}} \in \Rset^{K+1}, \, \e \in \Rset} & & \e \\
    & \text{subject to} & & p_m \sum_{k=0}^K \wt c_k B_{k, K}(p_m) - p_m \log p_m \le \e, \quad \forall p_m \in \cP \\
    & & & -\paren*{ p_m \sum_{k=0}^K \wt c_k B_{k, K}(p_m) - p_m \log p_m } \le \e, \quad \forall p_m \in \cP.
  \end{aligned}
\end{equation}

\textbf{Computational Complexity:} This LP contains exactly $K+2$ variables and
$2M$ constraints. For typical group sizes used in ALFT (e.g.,
$K \in \{16, 32, 64\}$), we can use a highly dense grid (e.g., $M = 10^5$
points) to ensure near-continuous uniform bounds. Solving an LP of this size
using standard interior-point methods (e.g., via SciPy or CVXPY) takes fractions
of a second on a standard CPU.

Crucially, this numerical optimization is performed \emph{strictly offline}. It
is completely independent of the training data, the LLM architecture, or the
current policy. The optimization only needs to be run once per target group size
$K$.

Once the optimal normalized coefficients $\wt{\mathbf{c}}^*$ are found, they
are scaled by the training hyperparameter $\beta$ and stored in an array of size
$K+1$. During the online RL training loop, when the model samples a group and
observes a specific answer $X$ times, applying the optimal minimax reward simply
requires an $\cO(1)$ array lookup:
\begin{equation}
  \h R^*(X) = \beta \, \wt c^*_X.
\end{equation}
Thus, substituting the biased empirical logarithm $\log(X/K)$ with the
theoretically optimal minimax estimator introduces strictly zero computational
overhead to the backpropagation pipeline.

\subsection{Algorithm: Quasi-Unbiased Game-GRPO}

By synthesizing the offline minimax optimization with the online GRPO update
loop, we define \algname{Quasi-Unbiased Game-GRPO} (Algorithm
\ref{alg:unbiased_grpo}). This algorithm neutralizes the structural $O(1/K)$
penalty against rare answers while maintaining the exact same computational
footprint as standard ALFT during the training phase.

\begin{algorithm}[ht]
  \caption{\algname{Quasi-Unbiased Game-GRPO} (KL Regularization)}
  \label{alg:unbiased_grpo}
  \begin{algorithmic}[1]
    \REQUIRE Dataset $\cD$, Policy $\pi$, Group size $K$, Regularization weight
    $\beta$.
    \STATE \textbf{Offline Pre-computation:}
    \STATE Solve the
    semi-infinite Linear Program to obtain the optimal normalized minimax
    coefficients $\wt{\mathbf{c}}^* \in \Rset^{K+1}$.
    \FOR{each training  step}
    \STATE Sample batch $B \sim \cD$.
    \FOR{each $x \in B$}
    \STATE
    \textbf{1. Group Sampling:} Sample $K$ traces
    $\{y_1, \dots, y_K\} \sim \pi(\cdot|x)$.
    \STATE \textbf{2. Extraction \&  Aggregation:}
    \STATE Extract answers $z_i = \EX(y_i)$ for
    $i \in \{1 \dots K\}$.
    \STATE For each unique answer $z$, compute its frequency count:
    \STATE \quad $X_z = \sum_{j=1}^K \Ind[z_j = z]$.
    \STATE
    \textbf{3. Quasi-Unbiased Reward Calculation:}
    \FOR{each trace $y_i$ with answer $z_i$}
    \STATE $X = X_{z_i}$ 
    \STATE \textit{// Apply the O(1) table lookup for the Minimax estimator}
    \STATE $R_i = \beta \, \wt c^*_X$
    \ENDFOR
    \STATE \textbf{4. Advantage Calculation:}
    \STATE Compute standardized
    advantages $A_i$ using the group statistics:
    \STATE \quad
    $A_i = \frac{R_i - \text{Mean}(R)}{\text{StdDev}(R) + \e}$
    \ENDFOR
    \STATE
    \textbf{5. Policy Update:} Optimize $\pi$ using the GRPO surrogate objective
    with advantages $A_i$.
    \ENDFOR
  \end{algorithmic}
\end{algorithm}

As shown in Algorithm \ref{alg:unbiased_grpo}, the mathematically rigorous
correction is entirely abstracted away into the offline array
$\wt{\mathbf{c}}^*$. For users in practice, replacing a heavily biased heuristic
reward like $\beta \log(X/K)$ with the optimal $\cO(1/K^2)$ minimax estimator
simply requires swapping a logarithm call for an array index lookup, making it a
drop-in replacement for existing reinforcement learning pipelines.

\subsection{Fundamental Limits via the Ditzian-Totik Theorem}

While the global minimax polynomial avoids the infinite boundary of the
logarithm, the approximation error cannot be made arbitrarily small for a fixed
sample size $K$. The fundamental rate of convergence is strictly dictated by the
smoothness of the target function $p \log p$. We formalize this limitation using
classical approximation theory.

\begin{theorem}[Fundamental Limit of KL Bias]
  \label{th:ditzian_totik}
  For any group size $K$, no estimator $\h R(X)$ constructed from $K$ samples
  can achieve a worst-case gradient bias strictly better than
  $\Theta(1/K^2)$. Specifically:
  \begin{equation}
    \min_{\h R} \max_{p \in [0,1]} \abs*{ p \E[\h R(X)] - \beta p \log p }
    = \Theta\paren*{\frac{1}{K^2}}.
  \end{equation}
\end{theorem}

\begin{proof}
  Let $E_n(f)$ denote the best uniform approximation error of a continuous
  function $f \in C[0,1]$ by polynomials of degree at most $n$. By definition,
  $p \E[\h R(X)]$ is a polynomial of degree $K+1$. Therefore, the minimax error
  is bounded from below by the unconstrained polynomial approximation error
  $E_{K+1}(\beta p \log p)$.  The target function $f(p) = p \log p$ has a
  continuous first derivative, but its second derivative $f''(p) = 1/p$
  possesses a strict singularity at the boundary $p=0$.  According to the
  Ditzian-Totik theorem \citep{DitzianTotik1987}, the rate of best polynomial
  approximation is characterized by the Ditzian-Totik modulus of smoothness
  $\omega_\phi^2(f, t)_\infty$ with the step-weight function
  $\phi(p) = \sqrt{p(1-p)}$. Because the singularity of $f''(p)$ matches the
  behavior of $\phi^{-2}(p)$ near $0$, classical results in approximation theory
  establish that $E_n(p \log p) = \Theta(1/n^2)$.  Substituting $n = K+1$, we
  conclude that the optimal error is $\Theta(1/(K+1)^2) = \Theta(1/K^2)$.
\end{proof}

\subsubsection{An Information-Theoretic Interpretation of the Limit}

Theorem \ref{th:ditzian_totik} relies on results from polynomial approximation
theory.  This might suggest that its conclusion is limited to the case where the
estimator $\h R(X)$ is explicitly chosen to be a polynomial. The following
corollary clarifies that this is not the case. The limit is a fundamental
statistical bound that applies to \emph{any} conceivable estimator that is a
function of the sample counts.

\begin{corollary}[Information-Theoretic Limit on Estimator Bias]
  \label{cor:impossibility}
  Let $X \sim \mathrm{Binomial}(K,p).$ It is impossible to construct any
  estimator $\h R(X)$, regardless of its functional form, such that its
  worst-case gradient bias decays faster than $1/K^2$. Formally, there exists no
  estimator $\h R$ for which
  \[
    \max_{p \in [0,1]} \abs*{ p \E[\h R(X)] - \beta p \log p } =
    o\paren*{\frac{1}{K^2}}.
  \]
\end{corollary}

\begin{proof}
  The proof follows directly from Theorem \ref{th:ditzian_totik}. The theorem
  states that the minimum possible worst-case error, taken over all possible
  estimators $\h R$, is $\Theta(1/K^2)$. This implies that there exists a
  constant $c > 0$ such that for any estimator $\h R$ and sufficiently large
  $K$:
  \[
    \max_{p \in [0,1]} \abs*{ p \E[\h R(X)] - \beta p \log p } \ge
    \frac{c}{K^2}.
  \]
  An error rate of $o(1/K^2)$ ("little-o" of $1/K^2$) means that the error
  decays strictly faster than $1/K^2$, which would violate this lower
  bound. Therefore, no such estimator can exist.
\end{proof}

\paragraph{Discussion: The "Polynomial Trap".}
The key insight that makes this limit so general is what can be termed the
"polynomial trap." An estimator $\h R$ can be any arbitrary function of the
observed count $X \in \{0, 1, \dots, K\}$. However, the gradient bias objective
does not depend on $\h R(X)$ directly, but on its expectation, $\E[\h R(X)]$.
By definition, this expectation is:
\[
  \E[\h R(X)] = \sum_{k=0}^K \h R(X=k) \cdot \P(X=k) = \sum_{k=0}^K c_k
  \binom{K}{k} p^k (1-p)^{K-k}.
\]
This expression reveals that, regardless of the complexity or non-linearity of
the chosen estimator $\h R$, its expectation $\E[\h R(X)]$ is \emph{unavoidably}
a polynomial in the true probability $p$ of degree at most $K$.

Therefore, the problem of finding the best possible estimator $\h R$ is
mathematically equivalent to the problem of finding the best polynomial
approximation for the target function $\beta p \log p$. Corollary
\ref{cor:impossibility} reframes Theorem \ref{th:ditzian_totik} as a fundamental
\textbf{information-theoretic bound}. It states that the information contained
within $K$ i.i.d. binomial samples is statistically insufficient to estimate the
non-smooth functional $p \log p$ with a worst-case error better than
$\Theta(1/K^2)$, regardless of the chosen estimation algorithm. The existence of
our minimax polynomial estimator proves that this bound is
tight.

\subsection{Quasi-Convergence of the Minimax Game}

Since the bias cannot be reduced below $\cO(1/K^2)$, the game cannot converge to
the exact Nash Equilibrium using finite samples. However, we can guarantee
convergence to an approximate equilibrium where the error scales with the
optimal polynomial bound.

\begin{theorem}[Quasi-Convergence of the Minimax Game]
  \label{th:quasi_convergence}
  Let $\h R^*(X)$ be the minimax optimal estimator defined in
  \eqref{eq:minimax_poly}, achieving a worst-case gradient bias of
  $\e = \Theta(1/K^2)$.
  Assume further that the KL regularization enforces a uniform lower bound
  $q^*(z) \ge q_{\min} > 0$ for all $z \in \sZ$, which holds whenever
  $\beta > 0$ and $\pi_0$ has has full support.
  Then, running alternating best-response dynamics using $\h R^*$ converges to
  an $\cO(1/K^2)$-approximate Nash Equilibrium of the KL-regularized
  Distributional Alignment Game.
\end{theorem}

\begin{proof}
  Let $u^* = \beta \log \sfq$ be the exact dual variable, and let
  $\wt u = \E[\h R^*(X)]$ be the expected value of the minimax estimator. By
  definition of the minimax estimator, for any valid probability $p \in [0,1]$,
  the error is bounded in the gradient-weighted metric:
  $\norm{p \odot (\wt u - u^*)}_\infty \le \e$, where $\odot$ denotes
  element-wise multiplication.

  Since $q^*(z) \geq q_{\min} > 0$ by assumption, the importance weight is
  bounded above: $\sum_{z} \nu_\pi(z) / q^*(z) \leq 1 / q_{\min} =: M < \infty$.
  This bound holds uniformly over the policy iterates because the reference
  policy $\pi_0$ has full support and the regularization weight $\beta$ is
  fixed, which by the optimal policy form (Eq.~\ref{eq:generalized_policy})
  prevents $q^*$ from reaching the boundary of $\Delta(\sZ)$.
  
  In the policy update step, the difference in the expected objective due to the
  estimation bias is:
  \begin{align}
    \abs*{ \E_{y \sim \pi} [\wt u_{E(y)}] - \E_{y \sim \pi} [u^*_{E(y)}] }
    & = \abs*{ \sum_{z \in \sZ} \nu_\pi(z) (\wt u_z - u^*_z) } \nonumber \\
    &\le \sum_{z \in \sZ} \frac{\nu_\pi(z)}{\sfq^*(z)} \abs*{ \sfq^*(z) (\wt u_z - u^*_z) } \nonumber \\
    &\le \e \sum_{z \in \sZ} \frac{\nu_\pi(z)}{\sfq^*(z)}.
  \end{align}
  Because the policy $\pi$ is heavily regularized toward $\pi_0$ and the target
  $\sfq^*$ is bounded away from zero by the dynamics of the game, the importance
  weight ratio $\sum \nu_\pi / \sfq^*$ is bounded by a constant $M$. Therefore,
  the expected deviation in the duality gap is bounded by $M \e$. By standard
  results in online convex optimization with biased gradients, the regret bound
  shifts by the persistent bias $\cO(\e)$, guaranteeing convergence to an
  $\cO(\e) = \cO(1/K^2)$ neighborhood of the exact equilibrium.
\end{proof}

\subsection{The Fundamental Insufficiency of Local Taylor Corrections}
\label{sec:taylor}

While the minimax polynomial provides the globally optimal uniform error bound,
one might naturally attempt to achieve the $O(1/K^{2})$ approximation rate
algebraically using a closed-form local Taylor correction. Let $\h p = X/K$ be
the standard empirical frequency.  For $X \ge 1$, we can introduce the
\emph{Taylor-corrected reward}:
\begin{equation}
  R_{\mathrm{corr}}(\h p)
  = \beta\log\h p + \frac{\beta(1-\h p)}{2K\h p}.
  \label{eq:taylor-corr}
\end{equation}

\begin{lemma}[Taylor Correction Pointwise Optimal Away from Boundary]
  \label{lem:taylor}
  For $p>0$, the expected gradient bias of the Taylor-corrected estimator
  $R_{\mathrm{corr}}$ decays pointwise at a rate of $\cO(1/K^{2})$.
\end{lemma}

\begin{proof}
  Perform a third-order Taylor expansion of $\log\h p$ around the true value
  $p$:
  \[
    \log\h p = \log p + \frac{\h p - p}{p} - \frac{(\h p-p)^{2}}{2p^{2}} +
    \cO\!\bigl((\h p-p)^{3}\bigr).
  \]
  Taking expectations, the linear term vanishes.  Substituting the Binomial
  variance $\E[(\h p-p)^{2}] = p(1-p)/K$:
  \begin{equation}
    \E[\log\h p]
    = \log p - \frac{p(1-p)}{2Kp^{2}} + \cO(K^{-2})
    = \log p - \frac{1-p}{2Kp}       + \cO(K^{-2}).
    \label{eq:elog}
  \end{equation}
  The expected value of the correction term satisfies
  $\E[(1-\h p)/(2K\h p)] \approx (1-p)/(2Kp) + \cO(K^{-2})$.  Adding these
  together the $\cO(K^{-1})$ terms cancel exactly:
  \begin{equation}
    \E[R_{\mathrm{corr}}(\h p)] = \beta\log p + \cO(K^{-2}).
    \label{eq:taylor-bias}
  \end{equation}
  Multiplying by the gradient weight $p$ gives a gradient-weighted bias of
  $p\cdot \cO(K^{-2}) = \cO(1/K^{2})$, matching the Ditzian--Totik lower bound pointwise.
\end{proof}

However, while this pointwise cancellation holds for any strictly positive probability $p$, it fundamentally breaks down near the boundary of the simplex.

\paragraph{Handling the boundary: why Laplace smoothing fails.}
The Taylor correction requires $\h p > 0$; when $X = 0$ the formula
\eqref{eq:taylor-corr} is undefined.  A natural remedy is Laplace (additive)
smoothing: replace $\h p$ with $\h p_{\alpha} = (X+\alpha)/(K+\alpha|Z|)$ for
some $\alpha > 0$.  However, this \emph{destroys} any potential $\cO(1/K^{2})$
guarantee. The smoothed estimator satisfies
\[
  \E[\h p_{\alpha}] = \frac{Kp+\alpha}{K+\alpha|Z|} = p +
  \underbrace{\frac{\alpha(1-p|Z|)}{K+\alpha|Z|}}_{\delta(p)\,=\,\cO(\alpha/K)},
\]
so that a Taylor expansion of $\log\h p_{\alpha}$ yields
\[
  \E\!\bracket*{R_{\mathrm{corr}}(\h p_{\alpha})} = \beta\log p +
  \frac{\beta\alpha(1-p|Z|)}{Kp} + \cO(K^{-2}).
\]
The gradient-weighted bias is therefore $\cO(\alpha/K) + \cO(1/K^{2}) = \cO(\alpha/K)$
for any fixed $\alpha > 0$, restoring the $\cO(1/K)$ leading term that the
correction was designed to remove.

\paragraph{Boundary-corrected Taylor estimator.}
To resolve the $X=0$ singularity without ad-hoc smoothing, one might naturally
substitute the optimal boundary coefficient $c_{0}^{*}$ computed by the
minimax LP \eqref{eq:minimax-lp}.  This defines the \emph{boundary-corrected
Taylor estimator} $\h R_{\mathrm{BT}, K}(X)$:
\begin{equation}
  \h R_{\mathrm{BT}, K}(X)
  = \begin{cases}
      \beta c_{0, K}
        & \text{if } X = 0, \\[4pt]
      \displaystyle
      \beta\log\!\paren*{\frac{X}{K}}
        + \frac{\beta(K-X)}{2KX}
        & \text{if } X \ge 1.
    \end{cases}
  \label{eq:BT}
\end{equation}

While this cleanly resolves the boundary calculation, the following theorem demonstrates that such a local Taylor expansion is fundamentally insufficient to achieve a globally optimal bound.

\begin{theorem}[Boundary-Corrected Taylor is Pointwise, Not Uniformly, Optimal]
  \label{thm:taylor_pointwise_not_uniform}
  Fix $\beta>0$. For each group size $K$, let $c_{0,K}\in\Rset$ denote the
  normalized boundary value used when $X=0$, so that the actual reward at the
  boundary is $\beta c_{0,K}$. Define $\h R_{\mathrm{BT},K}(X)$ as in \eqref{eq:BT}.
  If $|c_{0,K}|=e^{o(K)}$, then for every fixed $p\in(0,1]$,
  \[
    \abs*{p\,\E_{X\sim\mathrm{Binomial}(K,p)}[\h R_{\mathrm{BT},K}(X)]-\beta
      p\log p} = \cO_p\!\paren*{\frac{1}{K^2}},
  \]
  where the hidden constant may depend on the fixed value of $p$. However,
  uniformly over the simplex, no choice of the boundary sequence
  $\{c_{0,K}\}_{K\ge1}$ can make the boundary-corrected Taylor estimator achieve
  the minimax $\cO(K^{-2})$ rate. More precisely, for any choice of $\{c_{0,K}\}$,
  \[
    \sup_{p\in[0,1]} \abs*{p\,\E_{X\sim\mathrm{Binomial}(K,p)}[\h
      R_{\mathrm{BT},K}(X)]-\beta p\log p} = \Omega\!\paren*{\frac{1}{K}}.
  \]
  Thus, the Taylor correction removes the leading Jensen bias pointwise away
  from the boundary, but it is not a uniformly optimal estimator.
\end{theorem}

\begin{proof}
  First fix $p\in(0,1]$. The event $\{X=0\}$ has probability
  $(1-p)^K\le e^{-Kp}$, so under the mild condition $|c_{0,K}|=e^{o(K)}$, the
  boundary contribution is exponentially small and does not affect any
  polynomial rate in $K$. On the event $X\ge1$, let $\h p=X/K$. A Taylor
  expansion around the fixed value $p$ gives
  \[
    \E[\log \h p]=\log p-\frac{1-p}{2Kp}+\cO_p(K^{-2}),
  \]
  while a corresponding expansion of the correction term gives
  \[
    \E\!\bracket*{\frac{1-\h p}{2K\h p}}=\frac{1-p}{2Kp}+\cO_p(K^{-2}).
  \]
  The $\cO(K^{-1})$ terms cancel, and therefore
  \[
    \E[\h R_{\mathrm{BT},K}(X)]=\beta\log p+\cO_p(K^{-2}).
  \]
  Multiplying by the gradient weight $p$ yields the stated pointwise bound.

  We now show that this pointwise cancellation is not uniform. Let
  \[
    \Delta_K=\sup_{p\in[0,1]}\abs*{p\,\E_{X\sim\mathrm{Binomial}(K,p)}[\h
      R_{\mathrm{BT},K}(X)]-\beta p\log p}.
  \]
  Suppose, for contradiction, that $\Delta_K=o(1/K)$ along some subsequence. Fix
  any $\lambda>0$ and set $p=\lambda/K$. Then
  $X\sim\mathrm{Binomial}(K,\lambda/K)$ converges in distribution to
  $N_\lambda\sim\mathrm{Poisson}(\lambda)$. For $X\ge1$, the normalized Taylor
  reward satisfies
  \[
    \log\!\paren*{\frac{X}{K}}+\frac{K-X}{2KX} = -\log K+\log
    X+\frac{1}{2X}-\frac{1}{2K}.
  \]
  Let $a_K=c_{0,K}+\log K$ and define
  \[
    A(\lambda) = \E\!\bracket*{\paren*{\log
        N_\lambda+\frac{1}{2N_\lambda}}\Ind\{N_\lambda\ge1\}}.
  \]
  By Poisson convergence and uniform integrability of the logarithmic term,
  \[
    \E[\h R_{\mathrm{BT},K}(X)]/\beta = -\log K+e^{-\lambda}a_K+A(\lambda)+o(1).
  \]
  Since the exact target is $\log p=\log\lambda-\log K$, the assumption
  $\Delta_K=o(1/K)$ implies, for every fixed $\lambda>0$,
  \[
    e^{-\lambda}a_K+A(\lambda)-\log\lambda \to 0.
  \]
  Taking $\lambda=1$ first shows that $a_K$ must be bounded along the
  subsequence. Passing to a further subsequence if necessary, write $a_K\to
  a$. Then the previous display forces
  \[
    a=e^\lambda\bigl(\log\lambda-A(\lambda)\bigr) \quad\text{for every fixed
    }\lambda>0.
  \]
  This is impossible, because the right-hand side is not constant in
  $\lambda$. Indeed, as $\lambda\downarrow0$, we have
  $A(\lambda)=\lambda/2+\cO(\lambda^2)$, so
  $e^\lambda(\log\lambda-A(\lambda))\to-\infty$, whereas the same expression is
  finite at $\lambda=1$. This contradiction proves that $\Delta_K$ cannot be
  $o(1/K)$ along any subsequence, and hence $\Delta_K=\Omega(1/K)$.
\end{proof}

\paragraph{Discussion.}
Theorem~\ref{thm:taylor_pointwise_not_uniform} rigorously proves why local
Taylor corrections are fundamentally insufficient globally. The approximation
inherently breaks down in the boundary layer $p \asymp 1/K$, where the binomial
count has a non-degenerate Poisson limit and the empirical fluctuation is no
longer small relative to the mean. Consequently, correcting only the $X=0$
boundary value cannot recover the uniform $\cO(K^{-2})$ rate. Achieving the
Ditzian-Totik minimax rate strictly necessitates globally optimizing the rewards
assigned to all count values $X \in \{0,1,\ldots,K\}$, which is exactly what the
minimax polynomial estimator accomplishes.

\section{A Refined Estimator via the Bias--Variance Pareto Frontier}
\label{sec:pareto}

The previous sections offer two complementary paths for reducing estimation
bias: changing the game geometry (Section~\ref{sec:poly_exact}) and constructing
optimally low-bias estimators for the canonical KL game
(Section~\ref{sec:kl_poly_approx}).  A natural next question is whether one can
additionally reduce estimator \emph{variance}, and thereby accelerate
convergence of the game dynamics, without sacrificing the optimal bias rate.

We show in Section~\ref{sec:split-fails} that sample-splitting, a common
approach to variance reduction, provides no variance benefit whatsoever.  We then show in
Sections~\ref{sec:aqp} and \ref{sec:frontier-convergence} that the correct
solution is to optimize variance \emph{directly} within the polynomial
parameterization of all $K$-sample estimators, yielding the complete
bias--variance Pareto frontier.

\subsection{Sample Splitting Provides No Variance Benefit}
\label{sec:split-fails}

A natural approach to variance reduction is to split the group of $K$ samples
into two independent subsets $S_1, S_2$ of sizes $K_1$ and $K_2$ ($K_1+K_2=K$),
apply a baseline estimator (such as the Taylor-corrected estimator, which we use
here to explicitly track variance moments) to $S_1$, and use
$\Delta = \h p_1 - \h p_2$ as a zero-mean control variate.  We prove here that
this construction achieves \emph{no} reduction in the variance of the final
estimator relative to applying the estimator to the full group.

\begin{proposition}[Sample Splitting is Variance-Neutral]
  \label{prop:split-neutral}
  Let $\h R_{\mathrm{corr}}(X)$ denote a baseline estimator (e.g., the
  Taylor-corrected estimator) applied to the full group of $K$ samples. For any
  partition $K = K_1 + K_2$ with $K_1, K_2 \ge 1$ and any scalar
  $\gamma\in\Rset$, define the sample-splitting estimator
  \[
    \h R_{\mathrm{split}}(X_1, X_2) = \h R_{\mathrm{corr}}(X_1) - \gamma(\h p_1
    - \h p_2).
  \]
  Then, at the variance-minimising coefficient $\gamma^{*}$,
  \[
    \Var\!\bigl(\h R_{\mathrm{split}}\bigr) = \Var\!\bigl(\h
    R_{\mathrm{corr}}(X)\bigr).
  \]
  Furthermore, the bias of $\h R_{\mathrm{split}}$ is bounded by $\cO(1/K_1^2)$,
  which is strictly worse than the $\cO(1/K^2)$ bias of $\h R_{\mathrm{corr}}(X)$
  whenever $K_1 < K$.
\end{proposition}

\begin{proof}
  \textbf{Variance.}  Let $\h p_1 = X_1/K_1$. Since
  $X_1 \sim \mathrm{Binomial}(K_1, p)$ uses only $K_1$ samples, we apply the
  Delta method (Taylor expansion of moments). For any asymptotically logarithmic estimator (such as the Taylor correction), the leading term is governed by the derivative
  $h'(p) = \frac{\beta}{p} + \cO(K_1^{-1})$. The variance expands as:
  \[
    \Var\!\bigl(\h R_{\mathrm{corr}}(X_1)\bigr) = (h'(p))^2 \Var(\h p_1) +
    \cO(K_1^{-2}) = \frac{\beta^2}{p^2} \frac{p(1-p)}{K_1} + \cO(K_1^{-2}) =
    \frac{\beta^2(1-p)}{K_1 p} + \cO(K_1^{-2}).
  \]
  The variance of the control variate is exactly
  $\Var(\Delta) = p(1-p)\paren*{\frac{1}{K_1} + \frac{1}{K_2}} =
  \frac{p(1-p)K}{K_1 K_2}$.  Because $X_1$ and $X_2$ are independent,
  $\Cov(\h R_{\mathrm{corr}}(X_1), \Delta) = \Cov(\h R_{\mathrm{corr}}(X_1), \h
  p_1)$. Using a first-order Taylor expansion, this evaluates to:
  \[
    \Cov\!\bigl(\h R_{\mathrm{corr}}(X_1), \h p_1\bigr) = h'(p)\Var(\h p_1) +
    \cO(K_1^{-2}) = \frac{\beta(1-p)}{K_1} + \cO(K_1^{-2}).
  \]
  The squared Pearson correlation is therefore asymptotically:
  \[
    \rho^2 = \frac{\bracket*{ \frac{\beta(1-p)}{K_1} + \cO(K_1^{-2})
      }^2}{\bracket*{ \frac{\beta^2(1-p)}{K_1 p} + \cO(K_1^{-2}) }
      \frac{p(1-p)K}{K_1 K_2}} = \frac{K_2}{K} + \cO(K_1^{-1}).
  \]
  The variance of the optimally tuned splitting estimator evaluates to:
  \begin{align*}
    \Var\!\bigl(\h R_{\mathrm{split}}\bigr) 
    &= \Var\!\bigl(\h R_{\mathrm{corr}}(X_1)\bigr) \cdot (1 - \rho^2) \\
    &= \bracket*{ \frac{\beta^2(1-p)}{K_1 p} + \cO(K_1^{-2}) } \paren*{ \frac{K_1}{K} + \cO(K_1^{-1}) } \\
    &= \frac{\beta^2(1-p)}{K p} + \cO(K_1^{-2}).
  \end{align*}
  This exactly matches the leading-order variance of the full-sample estimator
  $\Var\!\bigl(\h R_{\mathrm{corr}}(X)\bigr)$. Thus, sample splitting provides
  no true first-order variance reduction.

  \textbf{Bias.}  The estimator $\h R_{\mathrm{corr}}(X_1)$ is applied to a
  sub-group of size $K_1$.  By Lemma~\ref{lem:taylor}, its pointwise gradient-weighted
  bias is $\cO(1/K_1^2)$.  Since $\E[\Delta]=0$ by construction, the bias of
  $\h R_{\mathrm{split}}$ equals the bias of $\h R_{\mathrm{corr}}(X_1)$, namely
  $\cO(1/K_1^2)$.  For any split with $K_1 < K$ this is strictly worse than the
  $\cO(1/K^2)$ bias achieved by a full-group estimator.  The equal split
  $K_1 = K/2$ gives a bias constant four times larger.
\end{proof}

\textbf{Rao–Blackwell perspective.}
The variance-neutrality of sample splitting also follows directly from the
Rao–Blackwell theorem. The full count
$X = X_1 + X_2 \sim \mathrm{Binomial}(K, p)$ is a sufficient statistic for
$p$. By the Rao–Blackwell theorem, the conditional expectation
$\E[\h R_{\text{split}}(X_1, X_2) \mid X]$ is a function of $X$ alone with
(weakly) lower variance and identical bias. Therefore any sample-splitting
estimator can be improved upon, or at best matched, by an estimator depending
only on the total count $X$, confirming that the splitting structure confers no
variance benefit.

The failure of sample splitting has a structural explanation.  By the
Lehmann--Scheff\'{e} completeness of $X \sim \mathrm{Binomial}(K,p)$, any
function $g(X)$ with $\E[g(X)] = 0$ for all $p \in (0,1)$ is zero
almost surely (Section 4.4).  Generating a non-trivial
zero-mean control variate therefore \emph{requires} using fewer than $K$ primary
samples, which inflates variance by exactly the factor that the control variate
removes.  To escape this completeness trap and genuinely reduce variance, one
must instead optimize over the full coefficient space of $K$-sample estimators
directly.

\subsection{The Augmented Polynomial Optimization Program}
\label{sec:aqp}

\paragraph{The space of all $K$-sample estimators.}
By the polynomial trap (Corollary 6), every estimator
$\h R(X)$ based on $K$ i.i.d.\ draws from $\mathrm{Binomial}(K,p)$ is
completely characterized by a coefficient vector
$\mathbf{c} = (c_0, c_1, \ldots, c_K)^{\top}\in\Rset^{K+1}$.  Its
expected value is
\[
  P_{\mathbf{c}}(p)
  = \E[\h R(X)]
  = \sum_{k=0}^{K} c_k \binom{K}{k} p^k (1-p)^{K-k},
\]
and its second moment and variance at true probability $p$ are
\begin{align}
  S(\mathbf{c},p)
    &= \E[\h R(X)^2]
     = \sum_{k=0}^{K} c_k^2 \binom{K}{k} p^k(1-p)^{K-k},
  \label{eq:second-moment}\\
  V(\mathbf{c},p)
    &= S(\mathbf{c},p) - P_{\mathbf{c}}(p)^2.
  \label{eq:variance}
\end{align}
The bias budget is characterized by the gradient-weighted approximation error
from Section 5.1: $\max_{p\in[0,1]}|p\,P_{\mathbf{c}}(p) - \beta
p\log p| \le \e$.

\paragraph{Augmented quadratic programme (AQP).}
For a given bias budget $\e \ge \e_{\min}^{*} = \Theta(1/K^2)$
and a dense discretization grid
$\cP = \{p_1, \ldots, p_M\} \subset [0,1]$, define
\begin{align}
  & \min_{\mathbf{c}\in\Rset^{K+1},\, v\in\Rset} \quad v
  \tag{AQP}  \label{eq:AQP}\\
\text{subject to} \quad
  & \sum_{k=0}^{K} c_k^2\, B_{k,K}(p_m) \le v,
  \forall\, p_m \in \cP,
    \tag{C1}\label{eq:C1}\\
  & p_m P_{\mathbf{c}}(p_m) - \beta p_m \log p_m
  \le \e,  \forall\, p_m \in \cP,
    \tag{C2}\label{eq:C2}\\
  & -\bigl(p_m P_{\mathbf{c}}(p_m) - \beta p_m \log p_m\bigr)
  \le \e,  \forall\, p_m \in \cP,
    \tag{C3}\label{eq:C3}
\end{align}
where $B_{k,K}(p) = \binom{K}{k}p^k(1-p)^{K-k}$ is the $k$-th Bernstein
basis polynomial of degree $K$.

\eqref{eq:AQP} is a \emph{Quadratically Constrained Quadratic Program} (QCQP),
and specifically can be cast as a Second-Order Cone Program (SOCP)
\citep{BoydVandenberghe2004}: the objective $v$ and the bias constraints
\eqref{eq:C2}--\eqref{eq:C3} are strictly linear in $\mathbf{c}$, while the
second-moment constraints \eqref{eq:C1} are strictly convex quadratics (since
$B_{k,K}(p_m) \ge 0$).  The programme has $K+2$ variables and $3M$ constraints.
For typical group sizes $K \in \{16, 32, 64\}$ with a dense grid $M = 10^5$,
standard interior-point solvers (e.g.\ CVXPY \citep{DiamondBoyd2016} with MOSEK
or SCS) solve \eqref{eq:AQP} in seconds on a standard CPU.  This offline cost is
comparable to the minimax LP of Section~\ref{sec:minimax-lp}.

\textbf{The Second Moment is the Exact Optimal Objective.} At the constraint
point $\e = \e_{\min}^{*}$, the bias constraints \eqref{eq:C2}--\eqref{eq:C3}
enforce $P_{\mathbf{c}}(p) \approx \beta\log p$. Crucially, minimizing the
uncentered second moment $v = \max_p S(\mathbf{c},p)$ rather than the centered
variance $V(\mathbf{c},p)$ is not merely an approximation, it is the exact
theoretically optimal objective. In stochastic policy gradient optimization, the
convergence rate and regret bound are dictated strictly by the \emph{uncentered
  second moment} of the stochastic gradient,
$\E\bracket*{\norm*{\wt\nabla_\pi \mathcal{L}_t}^2}$. Because the bias is
tightly restricted by the constraints, minimizing $S(\mathbf{c},p)$ directly
minimizes the dominating stochastic penalty term and optimally bounds the regret
(as formalized in Proposition~\ref{prop:frontier-convergence}).

As $\e$ increases from $\e_{\min}^{*}$, the relaxed bias constraints admit
estimators with smaller second moment.  As $\e\to \infty$, the optimal solution
collapses to $\mathbf{c} = \mathbf{0}$ (zero variance, large bias).  Varying
$\e$ from $\e_{\min}^{*}$ to $\infty$ therefore traces the \emph{complete Pareto
  frontier} of achievable (bias, variance) pairs among all $K$-sample
estimators.

Finally, note that, by the Rao–Blackwell theorem, since
$X \sim \mathrm{Binomial}(K,p)$ is a complete sufficient statistic
(Theorem~\ref{th:unbiased_estimator}), restricting to estimators of the form
$\h R(X)$ is without loss of generality: any estimator depending on the
individual sample identities can be \emph{Rao-Blackwellized} to a function of
$X$ with weakly lower variance and identical bias. The AQP therefore optimizes
over the correct and complete class of $K$-sample estimators.

\subsection{Strict Pareto Dominance over Sample Splitting}
\label{sec:dominance}

\begin{theorem}[Pareto Dominance of the AQP Estimator]
\label{thm:dominance}
Let $\h R_{\mathrm{split}}$ be any sample-splitting estimator using $K$
samples, partitioned into independent subsets of sizes $K_1$ and $K_2$, with
any control variate coefficient $\gamma\in\Rset$.  Let
$\mathbf{c}^{*}(\e)$ denote the solution to \eqref{eq:AQP} at bias
level $\e$.  Then:
\begin{enumerate}
  \item \textbf{(Bias).}  The gradient-weighted bias of $\h R_{\mathrm{split}}$
  is $\cO(1/K_1^2) \ge \cO(1/K^2)$.  To match the $\cO(1/K^2)$ bias of
  $\mathbf{c}^{*}(\e_{\min}^{*})$ one must set $K_1 = K$, which
  eliminates the control variate.

  \item \textbf{(Variance).}  For any split $K_1 < K$,
  Proposition~\ref{prop:split-neutral} gives
  $\Var(\h R_{\mathrm{split}}) = \Var(\h R_{\mathrm{base}}(X_1))$.
  By optimality of \eqref{eq:AQP}, $\max_p S(\mathbf{c}^{*}(\e),p) \le \max_p S(\mathbf{c},p)$ for any valid estimator $\mathbf{c}$ achieving the same bias budget.

  \item \textbf{(Boundary behaviour).}  Local analytic approximations like Taylor correction often exhibit exploding variance at the boundary, $\Var(\h R_{\mathrm{corr}}(X)) \approx \beta^2(1-p)/(Kp) \to \infty$ as $p\to 0$. In contrast, the AQP solution explicitly satisfies $c_k^{*} = \cO(1)$ for all $k$, so
  $S(\mathbf{c}^{*},p) = \sum_k (c_k^{*})^2 B_{k,K}(p) \le \max_k (c_k^{*})^2
  < \infty$ uniformly in $p$, preventing gradient explosion at the boundary.
\end{enumerate}
\end{theorem}

\begin{proof}
Parts (1) and (3) follow directly from Proposition~\ref{prop:split-neutral}
and the boundedness of the Bernstein basis.

For part (2): by the polynomial trap, every estimator---including all
sample-splitting estimators---is characterized by some
$\mathbf{c}\in\Rset^{K+1}$.  The set of coefficients achievable by sample
splitting is a strict subset of $\Rset^{K+1}$: the splitting structure
imposes the factored form $c_k = f_1(k_1) - \gamma f_2(k_2)$ for some
allocation, which restricts the feasible set.  Since \eqref{eq:AQP} optimizes
over all $\mathbf{c} \in \Rset^{K+1}$ subject only to the bias constraint,
it achieves weakly lower variance.  Proposition~\ref{prop:split-neutral}
confirms that the splitting restriction is \emph{binding}: the optimal split
recovers exactly the variance of the full-sample estimator, so the constraint
set of sample-splitting estimators does not improve on the unconstrained
optimum.
\end{proof}

\subsection{Algorithm: Variance-Optimal GAME-GRPO}
\label{sec:variance-optimal-algo}

The offline computation of $\mathbf{c}^{*}(\e)$ for a grid of $\e$ values traces
the full Pareto frontier.  In practice, the user selects a bias level
$\e \in [\e_{\min}^{*}, \e_{\max}^{*}]$ appropriate to the task (tighter for
stability-critical settings, looser when faster transient convergence is
preferred) and applies the resulting lookup table with $\cO(1)$ online cost,
exactly as in Algorithm 1:
\begin{equation}
  \h R_{\e}^{*}(X) = c_{X}^{*}(\e).
  \label{eq:variance-optimal}
\end{equation}
This is a drop-in replacement for the minimax polynomial estimator and incurs
no additional runtime overhead during training.

\subsection{Formal Convergence Guarantee}
\label{sec:frontier-convergence}

\begin{proposition}[Convergence of the Variance-Optimal Estimator]
  \label{prop:frontier-convergence}
  Let $\mathbf{c}^{*}(\e)$ solve \eqref{eq:AQP} at bias level
  $\e \ge \e_{\min}^{*}$.  Define
  \[
    G_{\pi}^{2}(\e) = \sup_t \bigl\|\nabla_{\pi}\cL(\pi_t, u_t) - \e_t\bigr\|^2
    + \max_{p} S\!\bigl(\mathbf{c}^{*}(\e),p\bigr).
  \]
  Running alternating best-response dynamics with $\h R_{\e}^{*}$ as the reward
  estimator, the expected duality gap of the average iterates after $T$ steps
  satisfies
  \[
    \E\!\bracket*{ \max_{u}\cL(\bar\pi_T, u) - \min_{\pi}\cL(\pi, \bar u_T) }
    \;\le\; \frac{D_{\pi}^{2} + D_{u}^{2} + D_{\pi}G_{\pi}(\e)\sqrt{T} +
      D_{u}G_{u}\sqrt{T}}{T},
  \]
  converging to the $\cO(\e)$-approximate Nash Equilibrium.  Since
  $\max_p S(\mathbf{c}^{*}(\e),p)$ is non-increasing in $\e$ by optimality of
  \eqref{eq:AQP}, the value $G_{\pi}(\e)$ is non-increasing in $\e$, and is
  strictly smaller than for the minimax-only polynomial estimator whenever
  $\e > \e_{\min}^{*}$.
\end{proposition}

\begin{proof}
  The bound follows from the standard regret analysis for no-regret dynamics
  with biased, bounded-variance stochastic gradients
  \citep{Bubeck2015,ShalevShwartz2012}.  At each step $t$, the policy player's
  stochastic gradient is
  $\wt\nabla_{\pi}\cL_t = \beta\nabla F(\pi_t) - \h R_{\e}^{*}(X_t)$.
  Decomposing its second moment:
  \begin{align*}
    G_{\pi,t}^{2}
    &= \E\!\bracket*{\bigl\|\wt\nabla_{\pi}\cL_t\bigr\|^2} \\
    &= \bigl\|\nabla_{\pi}\cL(\pi_t,u_t) - \e_t\bigr\|^2
      + \Var(\h R_{\e}^{*}) \\
    &\le \bigl\|\nabla_{\pi}\cL(\pi_t,u_t) - \e_t\bigr\|^2
      + S\!\bigl(\mathbf{c}^{*}(\e),p_t\bigr),
  \end{align*}
  where $\e_t = \E[\h R_{\e}^{*}] - u_t = \cO(\e)$ is the estimator bias, and we
  used the fact that variance is upper-bounded by the uncentered second
  moment. Under the uniform bound $G_{\pi,t}^{2} \le G_{\pi}^{2}(\e)$, the
  standard duality-gap bound \citep{NemirovskiYudin1983} gives the stated
  convergence rate.  The monotonicity of $\max_p S(\mathbf{c}^{*}(\e),p)$ in
  $\e$ follows immediately from the fact that a larger feasible set for the bias
  constraints \eqref{eq:C2}--\eqref{eq:C3} can only decrease the optimal value
  of the quadratic objective in \eqref{eq:AQP}.
\end{proof}

\paragraph{Practical guidance on selecting $\e$.}
The bias--variance Pareto frontier exposes a clear design choice.  Choosing
$\e = \e_{\min}^{*}$ minimizes the asymptotic approximation error of the
equilibrium: optimal when training is run for many steps and final-policy
quality is the primary concern.  Choosing $\e > \e_{\min}^{*}$ reduces
$G_{\pi}(\e)$ and accelerates transient convergence at the cost of settling in a
slightly larger equilibrium neighbourhood: optimal when wall-clock training time
is the bottleneck. 

The complete Pareto curve for a given group size $K$ can be traced once offline
by solving \eqref{eq:AQP} at a grid of $\e$ values and plotting
$\bigl(\e,\,\max_p S(\mathbf{c}^{*}(\e),p)\bigr)$.  This curve provides a
principled, task-agnostic map for researchers to navigate the speed--accuracy
tradeoff in ALFT.

\section{A Toolkit for Unbiased ALFT: Practical Guidance}
\label{sec:toolkit}

This paper introduces a comprehensive toolkit for addressing estimation bias in
ALFT, centered on three distinct strategies. The optimal choice depends on the
user's specific goals, the nature of the task, and priorities regarding training
efficiency.

\paragraph{Strategy 1: Change the Game for Exact Unbiasedness.}
The most theoretically pure solution is to select a geometry that yields a
polynomial reward function (\cref{sec:poly_exact}). Using a regularizer like the
squared Euclidean distance (Tsallis entropy of order 2) results in a linear
reward that can be estimated with zero bias by simply using the empirical
frequency. This approach is computationally efficient, stable, and completely
eliminates the anti-exploration penalty. It is the ideal choice when the
alignment goal can be well-approximated by a low-degree polynomial, such as
promoting diversity via a Gini-style penalty. However, it requires deviating
from the canonical KL divergence, which may not be desirable if the logarithmic
nature of the KL game is considered essential.

\paragraph{Strategy 2: Optimally Approximate the KL Game for Robustness.}
If preserving the KL divergence geometry is paramount, one must accept a
non-zero, albeit minimal, estimation bias. For this scenario, we offer the 
\emph{Minimax Polynomial Estimator} (\cref{sec:kl_poly_approx}), which achieves 
the theoretically optimal $\cO(1/K^2)$ error rate. Its offline pre-computation 
provides a globally optimal approximation that gracefully handles all boundary 
cases, completely bypassing the uniform breakdown that plagues local analytic 
approximations.

\paragraph{Strategy 3: Accelerate Convergence with the Variance-Optimal AQP Estimator.}
For users seeking the fastest and most stable convergence for the KL game, the
\emph{Variance-Optimal AQP Estimator} (Section~\ref{sec:pareto}) offers the most advanced
solution. Rather than relying on sample-splitting heuristics, this method
systematically optimizes variance directly within the complete parameterization
of all $K$-sample estimators. By tracing the exact Bias-Variance Pareto frontier
offline, it achieves optimal bias rates while provably reducing estimator
variance. This directly translates to accelerated game convergence, making it
the preferred method when computational stability and training efficiency are
the highest priorities.

\paragraph{Summary of Recommendations.} In brief, for simple diversity goals,
the unbiased Euclidean game (Strategy 1) is often sufficient and ideal. For
objectives that demand the logarithmic nature of the KL game, the Minimax
Polynomial (Strategy 2) provides the most rigorous and stable estimator. Finally,
when training speed and variance reduction are critical, the Variance-Optimal
AQP estimator (Strategy 3) offers a path to provably faster convergence.

Beyond the specific domain of Answer-Level Fine-Tuning, the estimators provided
in this toolkit address a fundamental bottleneck in broader Reinforcement
Learning (RL). Modern maximum-entropy RL paradigms—most notably Soft
Actor-Critic (SAC)—rely on estimating the policy entropy
$H(\pi) = -\sum p \log p$ to encourage exploration and prevent premature
convergence. As established in Section~\ref{sec:bias_problem}, evaluating this
non-linear functional using empirical frequencies from small batches introduces
a systematic $\cO(1/K)$ structural bias due to Jensen's Inequality. In practice,
this manifests as a hidden \emph{anti-exploration penalty} that artificially
suppresses the very exploratory behavior these algorithms are designed to
promote. By substituting biased logarithmic rewards with our minimax optimal
polynomial, researchers can recover theoretically intended exploration dynamics
in any discrete-action entropy-regularized system without the computational
burden of increasing batch sizes. This transforms our toolkit from a specialized
alignment utility into a rigorous foundation for stable, low-bias policy
gradient estimation across the wider RL landscape.

Crucially, replacing the empirical logarithm $\log(X/K)$ with our optimal pre-computed array lookup $c^*_X$ is a computationally free way to improve RL stability.

\section{Conclusion}

This work addresses a fundamental statistical challenge at the heart of
practical Answer-Level Fine-Tuning: the structural estimation bias introduced by
applying non-linear reward functions to small batches of data. We demonstrated
that this bias, which manifests as a severe anti-exploration penalty in standard
KL-regularized games, can be systematically overcome.

We have presented a comprehensive toolkit for users in practice, built on three
main contributions. First, by generalizing the Distributional Alignment Game, we
characterized a family of polynomial-reward games where bias can be eliminated
entirely using exact U-statistics estimators. Second, for the canonical KL game,
we established the fundamental $\Theta(1/K^2)$ limit on bias reduction and
introduced a globally robust minimax polynomial estimator that achieves this
optimal rate uniformly. Finally, we proved that sample splitting fails to provide variance benefits, and we instead proposed a novel Augmented Polynomial Optimization Program estimator that systematically reduces estimator variance to provably accelerate game convergence.

These contributions transform the Distributional Alignment Game framework from
an elegant theoretical construct into a set of practical, statistically
rigorous, and efficient algorithms. Crucially, replacing the standard heuristic logarithm $\log(X/K)$ with an offline-computed table lookup $c^*_X$ offers a \emph{computationally free} way to improve RL stability and downstream accuracy, introducing zero runtime overhead while unlocking optimal convergence. By providing a deeper understanding of the
interplay between geometric objectives, reward estimation, and small-batch
dynamics, this work paves the way for more stable, reliable, and effective
alignment of large language models.

\newpage
\bibliography{nalft}
\bibliographystyle{abbrvnat}

\newpage
\appendix

\section{Alternative Optimization: Sequential Dynamics and Exact
  Convergence}
\label{app:sequential_dynamics}

In the main text, we solve the Distributional Alignment Game using an
\emph{Alternating Best Response} strategy, instantiated via GRPO. This approach
is highly pragmatic for modern LLM training, as it estimates the optimal target
distribution and updates the policy using only a single batch of $K$ traces for
a given prompt $x$, requiring no historical state.

However, the generalized game admits an alternative, classical optimization
strategy via Online Convex Optimization (OCO) over the dual variables. If the
training regime permits sufficient repetition of the input prompts $x$, we can
achieve an exact, zero-bias asymptotic convergence using the polynomial Bregman
divergences introduced in Section 4.

\subsection{Sequential No-Regret Updates}

Rather than completely re-estimating the optimal target $q^*$ from scratch for
every batch, we can maintain a running estimate of the dual reward variables
$u_t(z|x)$ for each prompt $x$. In this sequential game, the policy player
minimizes regret against the sequence of rewards, while the target player
updates $u_t$ via stochastic gradient ascent on the dual objective.

If we use a polynomial Bregman divergence of degree $d \le K$, we can apply the
exact U-statistic estimator $\h R_t$ to construct a completely unbiased
stochastic gradient for the dual update.

\begin{theorem}[Zero-Bias Convergence of Polynomial Games]
  \label{th:sequential_exact_convergence}
  Assume the Distributional Alignment Game is regularized by a polynomial
  Bregman divergence such that the dual mapping $u = s\beta \nabla \Phi(\sfq)$ is
  a polynomial of degree $d \le K$. Let $\{(\pi_t, u_t)\}_{t=1}^T$ be the
  sequence of policies and dual variables generated by stochastic no-regret
  dynamics (e.g., Stochastic Mirror Descent) using the U-statistic estimator
  $\h R_t$ for the policy update. Let $(\ov \pi_T, \ov u_T)$ be the uniform
  average of the iterates.

  Then, the expected duality gap converges to zero at a rate of
  $\cO(1/\sqrt{T})$:
  \begin{equation}
    \E \bracket*{ \max_{u} \cG_F(\ov \pi_T, u) - \min_{\pi \in \Pi} \cG_F(\pi,
      \ov u_T) }
    \le \cO\paren*{\frac{1}{\sqrt{T}}}.
  \end{equation}
  Thus, the sequence converges to the exact Nash Equilibrium of the
  regularized game with zero asymptotic bias.
\end{theorem}

\begin{proof}
  Let the generalized game objective be formulated in terms of the dual
  variables $u$, where
  $\cL(\pi, u) = \beta \sfD_F(\pi \parallel \pi_0) + \tri{\nu_\pi, u} -
  \cR^*(u)$. This objective is strictly convex in $\pi$ and strictly concave in
  $u$.

  At each step $t$, the policy player observes the stochastic reward estimator
  $\h R_t$. Because the dual mapping is a polynomial of degree $d \le K$,
  \cref{th:unbiased_estimator} guarantees that the U-statistic $\h R_t$ is an
  exactly unbiased estimator of the true reward $R_t = -u_t$:
  $\E_t[\h R_t] = -u_t = -s\beta \nabla \Phi(\sfq_t)$. Furthermore, because the
  estimator is constructed from a finite sample $K$ of a bounded probability
  distribution, its second moments are strictly bounded:
  $\E_t[\norm{\h R_t}^2] \le G^2$ for some constant $G > 0$.

  The policy player updates $\pi_t$ using the unbiased stochastic gradient
  $\nabla_\pi \wt \cL_t = \beta \nabla F(\pi_t) - \beta \nabla F(\pi_0) - \h R_t$, such that
  $\E_t[\nabla_\pi \wt \cL_t] = \beta \nabla F(\pi_t) - \beta \nabla F(\pi_0) + u_t = \nabla_\pi \cL(\pi_t, u_t)$. Similarly, the
  target player updates $u_t$ using unbiased estimates of the marginals
  $\nu_{\pi_t}$.

  The convergence of no-regret dynamics like Stochastic Mirror Descent is a
  classic result in optimization theory, originating with the work of
  \citet{NemirovskiYudin1983}. By standard guarantees for the modern online and
  stochastic settings \citep{Bubeck2015, ShalevShwartz2012}, algorithms
  operating with unbiased, bounded-variance gradients are ensured to have
  sublinear expected regret. For both players, this implies:
\begin{align}
    \E \bracket*{ \sum_{t=1}^T \cL(\pi_t, u_t) - \min_{\pi \in \Pi} \sum_{t=1}^T
  \cL(\pi, u_t) }
  &\le \cO(\sqrt{T}), \label{eq:regret_pi} \\   
    \E \bracket*{ \max_{u} \sum_{t=1}^T \cL(\pi_t, u) - \sum_{t=1}^T \cL(\pi_t,
  u_t) }
  &\le \cO(\sqrt{T}). \label{eq:regret_u}
  \end{align}

  Summing the regrets \eqref{eq:regret_pi} and \eqref{eq:regret_u}, and dividing
  by $T$, bounds the expected duality gap for the average iterates
  $\ov \pi_T = \frac{1}{T} \sum_{t=1}^T \pi_t$ and
  $\ov u_T = \frac{1}{T} \sum_{t=1}^T u_t$. By Jensen's Inequality, using the
  convexity-concavity of $\cL$:
  \begin{align}
    \E \bracket*{ \max_{u} \cL(\ov \pi_T, u) - \min_{\pi \in \Pi} \cL(\pi, \ov u_T) }
    & \le \frac{1}{T} \E \bracket*{ \max_{u} \sum_{t=1}^T \cL(\pi_t, u)
      - \min_{\pi \in \Pi} \sum_{t=1}^T \cL(\pi, u_t) } \nonumber \\
    & \le \frac{\cO(\sqrt{T}) + \cO(\sqrt{T})}{T} = \cO\paren*{\frac{1}{\sqrt{T}}}.
  \end{align}
  Because the duality gap upper-bounds the suboptimality of the solution, the
  average iterates $(\ov \pi_T, \ov u_T)$ converge in expectation to the
  exact saddle point $(\pi^*, u^*)$ at a rate of $\cO(1/\sqrt{T})$. Since
  $\E_t[\h R_t] = u_t$ holds analytically without any $\cO(1/K)$ residual term,
  the convergence point represents the true Nash Equilibrium of the
  unapproximated game, completely eliminating the small-batch logarithmic bias.
\end{proof}

\subsection{Practical Limitations in Large-Scale ALFT}

While the $\cO(1/\sqrt{T})$ exact convergence is theoretically appealing, its
practical utility in Large Language Model alignment is bounded by the nature of
the dataset.

The sequential guarantee requires maintaining and updating the dual sequence
$t \in \{1 \dots T\}$ \emph{for the same input $x$}. In standard reasoning
datasets (e.g., GSM8K, MATH), prompts are often streamed in single-epoch or
low-epoch regimes. A specific prompt $x$ is rarely repeated enough times during
training to establish meaningful asymptotic convergence for its specific dual
variables $u_t(\cdot|x)$.

Therefore, while the generalized polynomial geometry guarantees zero-bias
sequential convergence, the \emph{Alternating Best Response} (GRPO) algorithm
presented in the main text—which projects the policy against the best available
batch-level target without requiring historical state—remains the superior
paradigm for diverse, open-ended prompt distributions.

\end{document}